\documentclass{article}

\usepackage[margin = 1.35in]{geometry}  

\usepackage[numbers, sort&compress]{natbib}
\usepackage{multicol}
\usepackage{multirow}
\usepackage[bookmarks=true]{hyperref}

\usepackage{amsmath, amssymb, amsthm}

\usepackage{subcaption}
\usepackage{graphicx}

\usepackage{algorithm}
\usepackage{algpseudocode}
\makeatletter
\let\OldStatex\Statex
\renewcommand{\Statex}[1][0]{%
  \setlength\@tempdima{\algorithmicindent}%
  \OldStatex\hskip\dimexpr#1\@tempdima\relax}
\makeatother

\algnewcommand\algorithmicinput{\textbf{Input:}}
\algnewcommand\Input{\item[\algorithmicinput]}
\algnewcommand\algorithmicoutput{\textbf{Output:}}
\algnewcommand\Output{\item[\algorithmicoutput]}

\usepackage{authblk}

\usepackage{tikz}
\usetikzlibrary{positioning}

\usepackage{accents}  %
\newcommand{\ubar}[1]{\underaccent{\bar}{#1}}

\newcommand{\R}{\mathbb{R}}

\DeclareMathOperator{\Sym}{Sym}
\def \PSD{\mathbb{S}}

\DeclareMathOperator{\tr}{tr}

\DeclareMathOperator{\Diag}{Diag}
\DeclareMathOperator \vect{vec}

\def \transpose{^\mathsf{T}}
\def \inv{^{-1}}

\def \pinv {^\dagger}

\DeclareMathOperator{\image}{image}

\DeclareMathOperator{\Graph}{graph}

\DeclareMathOperator{\Stiefel}{St}

\DeclareMathOperator{\proj}{Proj}

\DeclareMathOperator{\Orthogonal}{O}
\DeclareMathOperator{\SO}{SO}

\DeclareMathOperator{\SE}{SE}

\DeclareMathOperator{\Gaussian}{\mathcal{N}}  %
\DeclareMathOperator{\Langevin}{Langevin}  %

\def \Graph {\mathcal{G}}  %
\def \Nodes {\mathcal{V}}  %
\def \Edges {\mathcal{E}}  %
\newcommand{\directed}[1]{\vec{#1}}  %
\def \dEdges{\directed{\Edges}}  %

\def \edge{\lbrace i,j \rbrace}  %
\def \dedge{(i,j)}  %

\def \incEdges{\delta}

\def \outEdges{\delta^{-}}

\def\Lap{L}  %

\theoremstyle{definition}
\newtheorem{problem}{Problem}

\DeclareMathOperator{\BDiag}{BlockDiag}
\DeclareMathOperator{\SymBlockDiag}{SymBlockDiag}
\DeclareMathOperator{\SBD}{SBD}  %

\def \pose{x}
\def \tran{t}
\def \rot{R}

\newcommand{\true}[1]{\ubar{#1}}
\newcommand{\noisy}[1]{\tilde{#1}}

\def \tpose{\true{\pose}}
\def \ttran{\true{\tran}}
\def \trot{\true{\rot}}

\def \optsym{*}
\def \initsym{{(0)}}

\def \Rinit{\rot^\initsym}

\def \Rhat{\hat{\rot}}
\def \Ropt{\rot^\optsym}

\def \npose{\noisy{\pose}}
\def \ntran{\noisy{\tran}}
\def \nrot{\noisy{\rot}}

\def \rotsym{\rho}  %
\def \transym{\tau}  %

\def \TranW{W^{\transym}}  %
\def \LapTranW{\Lap(\TranW)}  
\def \RotW{W^{\rotsym}}  %
\def \LapRotW{\Lap(\RotW)}  

\def \MeasRotGraph{\noisy{G}^\rotsym}  %
\def \TrueRotGraph{\true{G}^\rotsym}   %
\def \MeasRotConLap{\Lap(\MeasRotGraph)}  %
\def \TrueRotConLap{\Lap(\TrueRotGraph)}  %

\def \RotConLapDev{\Delta \MeasRotConLap}  %

\def \nCrossTerms{\noisy{V}}  %
\def \nOuterProducts{\noisy{\Sigma}}  %

\def \tQ{\true{Q}}
\def \nQ{\noisy{Q}}  %
\def \dQ{\Delta Q}

\def \nQtran{\noisy{Q}^\transym}  %

\def \tQtran{\true{Q}^\transym}  %

\def \Oorbdist{d_{\mathcal{O}}}
\def \Sorbdist{d_{\mathcal{S}}}

\def \eigvecs{Y^{\optsym}}
\def \projsym{\Pi}

\newcommand{\SOrounded}[1]{\projsym_{\mathcal{S}}(#1)}

\def \moreiraR{\tilde{M}}
\def \moreiratR{\ubar{M}}
\def \moreiraL{\mathcal{L}}

\def \Cube{Cube}

\def \Sphere{Sphere}
\def \Grid{Grid}
\def \Torus{Torus}
\def \Garage{Garage}

\def \eigvecSpace{\mathcal{Y}}

\theoremstyle{plain}
\newtheorem{thm}{Theorem}
\newtheorem{cor}[thm]{Corollary}
\newtheorem{lem}[thm]{Lemma}

\theoremstyle{definition}
\newtheorem{definition}[thm]{Definition}

\theoremstyle{remark}

\def \CSAIL{Computer Science and Artificial Intelligence Laboratory}
\def \LIDS{Laboratory for Information and Decision Systems}
\def \MIT{Massachusetts Institute of Technology}
\def \MITaddr{Cambridge, MA 02139, USA}

\author{Kevin J. Doherty\thanks{Corresponding author. Email:
    \href{mailto:kjd@csail.mit.edu}{{\tt kjd@csail.mit.edu}}}}
\affil{\CSAIL, \MIT, \MITaddr}

\title{Performance Guarantees for Spectral Initialization in Rotation Averaging and Pose-Graph SLAM}

\author[2]{David M.\ Rosen}
\affil{\LIDS, \MIT, \MITaddr}

\author[1]{John J.\ Leonard}
\date{}

\begin{document}
\maketitle

\begin{abstract}
  In this work we present the first initialization methods equipped with
  \emph{explicit performance guarantees} that are adapted to the pose-graph
  simultaneous localization and mapping (SLAM) and rotation averaging (RA)
  problems. SLAM and rotation averaging are typically formalized as large-scale
  nonconvex point estimation problems, with many bad local minima that can
  entrap the smooth optimization methods typically applied to solve them; the
  performance of standard SLAM and RA algorithms thus crucially depends upon the
  quality of the estimates used to initialize this local search. While many
  initialization methods for SLAM and RA have appeared in the literature, these
  are typically obtained as purely heuristic approximations, making it difficult
  to determine whether (or under what circumstances) these techniques can be
  reliably deployed. In contrast, in this work we study the problem of
  initialization through the lens of \emph{spectral relaxation}. Specifically,
  we derive a simple spectral relaxation of SLAM and RA, the form of which
  enables us to exploit classical linear-algebraic techniques (eigenvector
  perturbation bounds) to control the distance from our spectral estimate to
  \emph{both} the (unknown) ground-truth \emph{and} the global minimizer of the
  estimation problem as a function of measurement noise. Our results reveal the
  critical role that spectral graph-theoretic properties of the measurement
  network play in controlling estimation accuracy; moreover, as a by-product of
  our analysis we obtain new bounds on the estimation error for the
  \emph{maximum likelihood} estimators in SLAM and RA, which are likely to be of
  independent interest. Finally, we show experimentally that our spectral
  estimator is very effective in practice, producing initializations of
  comparable or superior quality at lower computational cost compared to
  existing state-of-the-art techniques.
\end{abstract}

\newpage

\tableofcontents

\newpage

\section{Introduction}\label{sec:intro}

Simultaneous localization and mapping (SLAM) is the process by which a robot
jointly infers its pose and the location of environmental landmarks; this is a
fundamental capability of mobile robots, supporting navigation, planning, and
control \cite{rosen2021advances}.  State-of-the-art methods typically formalize 
SLAM and rotation averaging as large-scale M-estimation problems, and then apply smooth 
first- or second-order local optimization methods to efficiently recover a point estimate.
However, the fact that robot orientations are elements of the special orthogonal group $\SO(d)$, a nonconvex set, makes these estimation problems inherently \emph{nonconvex}, with many bad local minima that can entrap the local optimization methods commonly applied to solve them.  The performance of standard SLAM and RA algorithms thus crucially depends upon the quality of the estimates used to initialize the local search.
In consequence, a great deal of prior work has been dedicated to the development of initialization techniques (see
\citet{carlone2015initialization} for a review). While many of these techniques often work well in practice, the fact that they are obtained as heuristic approximations makes it difficult to ascertain \emph{what specific features} of SLAM or RA problems determine their performance. As a result, it is difficult to say when, or under
what conditions, these techniques can be \emph{reliably} deployed.

In this work, we propose a simple \emph{spectral initialization} method for
pose-graph SLAM and rotation averaging that we prove enjoys \emph{explicit
  performance guarantees}. To the best of our knowledge, these are the first
concrete guarantees to appear in the literature for any initialization technique
adapted to these applications. Our analysis gives direct control over the
estimation error of a spectral initialization in terms of the spectral
properties of the measurement network.\footnote{Recent work has identified
  spectral properties of measurement networks as key quantities controlling the
  performance of estimators for these problems, though this connection
  (particularly in the context of SLAM) remains under-explored (see
  \cite{rosen2021advances} for a recent review).} This allows us to control the
distance from the spectral estimate to the global minimizer of the estimation
problem; this is critical for ensuring that the initialization lies in the
locally convex region around the global minimizer, and therefore that this
minimizer can be recovered by a subsequent local refinement (see Figure
\ref{fig:diagram}). Our proof of this result relies on new estimation error
bounds for the global minimizers (i.e.\ the \emph{maximum likelihood}
estimators) of SLAM and rotation averaging problems, which are likely to be of
independent interest. Algorithmically, our approach only requires computing the
first few eigenpairs of a symmetric matrix, which can be achieved using any
off-the-shelf implementation of the Lanczos method (e.g. the MATLAB
\texttt{eigs} command). Our empirical results on both synthetic data and
standard pose-graph SLAM benchmarks demonstrate that the spectral estimator
typically performs far better than our worst-case analysis suggests, achieving
solution quality and computation times competitive with state-of-the-art
approaches. Beyond its utility as an \emph{initialization method} for
M-estimation, our results show that spectral relaxation provides an inexpensive
method for rotation averaging and pose-graph optimization in its own right
(i.e.\ \emph{without} the need to perform subsequent nonconvex optimization or
semidefinite relaxation) that attains an asymptotic error bound comparable to
the (globally optimal) M-estimator, and provides near-optimal estimates in
practice.

The remainder of the paper proceeds as follows: In Section
\ref{sec:related-work}, we discuss related literature on robot perception and
rotation averaging. Section \ref{sec:problem-formulation} formalizes the estimation problem, and Section \ref{sec:spectral-init} describes our
spectral initialization procedure. In Section \ref{sec:main-results} we present
our main results: an analysis controlling the estimation error of both the
spectral initialization  and the global minimizer for the
rotation averaging and pose-graph SLAM problems, as well as a bound on the
distance between the spectral initialization and the globally optimal solution.
Section \ref{sec:experimental-results} demonstrates the empirical performance of
our spectral estimator on benchmark SLAM datasets, together with
our performance bound evaluated on synthetic data.  These results show, in particular, that
the spectral estimator is competitive with state-of-the-art techniques for
initialization.

\begin{figure}[t]
  \centering
  \includegraphics[width=0.6\columnwidth]{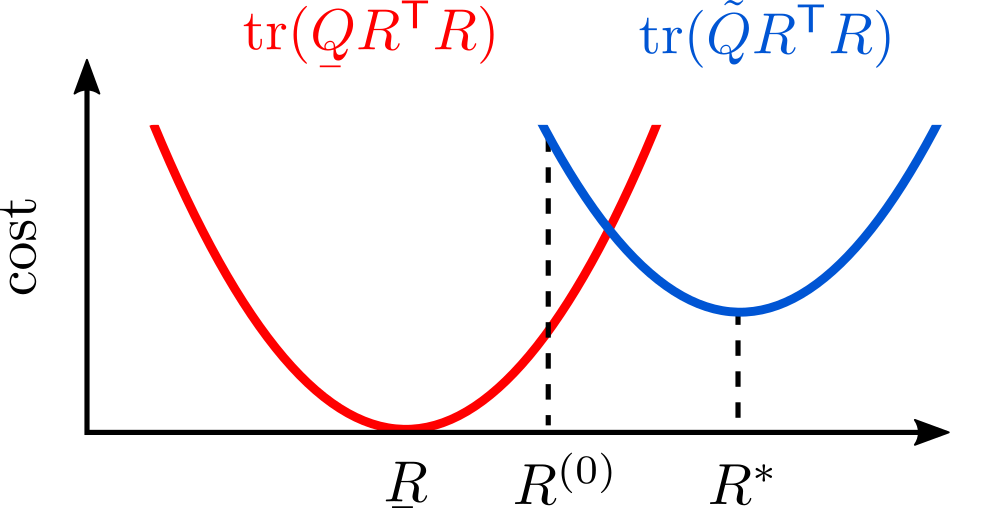}
  \caption{\textbf{Comparing true, optimal, and initial rotation estimates.} We
    are interested in bounds on the deviation of an initial estimate $\Rinit$
    from the (latent) ground truth $\protect\trot$ and the globally optimal
    solution $\Ropt$. \label{fig:diagram}}
\end{figure}

\section{Related work}\label{sec:related-work}

Simultaneous localization and mapping and rotation averaging problems are often
formulated as high-dimensional, nonconvex optimization problems. Consequently,
solving these problems typically requires efficient algorithms for producing an
``initial guess.'' Historically, research on this topic has focused on
developing cheap, but typically inexact, convex or linear relaxations of the
SLAM (resp.\ rotation averaging) problems (e.g.\
\cite{carlone2015initialization, martinec2007robust}). While these techniques
often work well in practice, the fact that they are obtained as heuristic
approximations makes it difficult to ascertain \emph{what specific features} of
SLAM or RA problems determine their performance. Consequently, it is difficult
to assess \emph{under what conditions} these techniques can be reliably
deployed.

A related line of research is the development of Cram\'er-Rao bounds for the
pose-graph SLAM and rotation averaging problems
\cite{boumal2014cramer,khosoussi2014novel,chen2021cramer}; these works provide
\emph{lower bounds} on the achievable estimation error \emph{in expectation}. In
this work, we derive a complementary set of \emph{upper bounds} on the
estimation error on a \emph{per instance} basis. Interestingly, our estimation
error \emph{upper} bounds depend upon precisely the same spectral quantities as
do the Cram\'er-Rao (\emph{lower}) bounds, indicating that graph spectra are
objects of central importance in understanding the statistical properties of
SLAM and RA estimators.

The spectral relaxation approach to initialization that we consider has
previously appeared in other problem settings, particularly in the area of phase
synchronization problems (cf. \cite{boumal2013robust, ling2020near,
  boumal2016nonconvex, singer2011angular}). In particular, \citet{ling2020near}
describe error bounds that are qualitatively similar to those described in this
paper, though theirs are concerned specifically with \emph{orthogonal} group
synchronization problems. \citet{liu2020unified} take a similar approach to ours
in order to derive error bounds for spectral estimators of synchronization
problems defined over subgroups of the orthogonal group (including $\SO(d)$),
but employ a different definition of the perturbation than the one we consider
here. As we will show, our notion of perturbation has the advantage that it
follows naturally from a generative model of SLAM and RA, and furthermore,
directly reveals the spectral properties of the measurement network
(specifically, a kind of generalized algebraic connectivity) as the key
quantities controlling the worst-case performance of our spectral initialization
method.

Recently, \citet{moreira2021fast} proposed a computationally-efficient
Krylov-Schur decomposition approach for pose-graph SLAM. We show in Appendix
\ref{app:moreira-comparison} that their method is formally equivalent to a
\emph{special case} of our approach (namely, an \emph{unweighted},
\emph{rotation-only} variant of our spectral initialization procedure). However,
our construction arises more naturally from spectral relaxation, and
additionally allows for the incorporation of translational measurements, which
we show in Section \ref{sec:experimental-results} can have a significant impact
on estimation quality. \citet{arrigoni2016spectral} also describe a spectral
method for $\SE(d)$-synchronization. While an analysis similar to ours could
likewise be carried out for their method, the form of the relaxation they
consider would lead to more complicated bounds due to a dependence on the scale
of the translational states. Finally, \citet{boots2013spectral} consider
spectral techniques for the range-only SLAM problem. Though their problem
setting differs from the one considered here, extension of the techniques
presented in this work to scenarios with different types of measurement models
is an interesting area for future work.

Finally, \emph{certifiably-correct} machine perception has emerged as a key area of
interest to the robotics community, resulting in the development of algorithms
capable of directly computing \emph{globally optimal} solutions of certain nonconvex
estimation problems under moderate noise \cite{carlone2015lagrangian,
  carlone2016planar, rosen2019se, fan2020cpl, briales2017cartan,
  dellaert2020shonan, tian2019distributed}. Our analysis provides new bounds on the estimation error of the maximum likelihood estimators recovered by these techniques in terms of the magnitude of the
measurement noise. Moreover, the bounds we present suggest that when these
estimators, which are often based on large-scale semidefinite relaxations, do
 attain globally optimal solutions, the resulting estimates have error
bounds that \emph{match} (up to small constant factors) the error bounds we
derive for our spectral initialization, which is easily implemented and computationally inexpensive.

\section{Preliminaries and formulation}\label{sec:problem-formulation}

\subsection{Notation and preliminaries}

\paragraph{Lie groups and matrix manifolds:}

We will make use of the matrix realizations of several Lie groups, most
prominently the $d$-dimensional special Euclidean and special orthogonal groups,
denoted $\SE(d)$ and $\SO(d)$, respectively. $\SE(d)$ can be realized as a
matrix group according to:
\begin{equation}
  \SE(d) \triangleq \left\{ \begin{bmatrix} R & t \\ 0 & 1\end{bmatrix} \in \R^{(d+1)\times (d+1)} \mid R \in \SO(d),\ t \in \R^d \right\},
\end{equation}
and the group $\SO(d)$ can be realized as:
\begin{equation}
  \SO(d) \triangleq \left\{ R \in \R^{d\times d} \mid R^TR = I_d,\ \det(R) = 1\right\},
\end{equation}
where $I_d$ is the $(d \times d)$ identity matrix. The \emph{Stiefel manifold}
$\Stiefel(k, n)$ is the set of orthonormal $k$-frames in $\R^n (k \leq n)$:
\begin{equation}
  \Stiefel(k, n) \triangleq \left\{ V \in \R^{n \times k} \mid V\transpose V = I_k \right\}.  \label{eq:stiefel-def}
\end{equation}

\paragraph{Linear algebra:}

For a symmetric matrix $S$, $S \succeq 0$ denotes that $S$ is
positive-semidefinite. The eigenvalues of a symmetric matrix $S \in \R^{n \times
  n}$ are denoted $\lambda_{1}(S) \leq \lambda_{2}(S) \leq \ldots \leq
\lambda_{n}(S)$. We will also consider several block-structured matrices, and
make use of a few special operators acting on them. Following the notation of
\citet{rosen2019se}, given square matrices $A_i \in \R^{d \times d}, i = 1,
\ldots, n$, we let $\Diag(A_1, \ldots. A_n)$ denote the matrix direct sum (i.e.,
the block-diagonal matrix having $A_1, \ldots, A_n$ as its diagonal blocks).
Furthermore, given a block-structured matrix $B$, let $\BDiag_{d}(B)$ denote the
operator extracting a $d \times d$ block-diagonal matrix from $B$. Finally, let
$\SBD(d, n)$ denote the set of $dn \times dn$ symmetric block-diagonal matrices
with diagonal blocks of size $d \times d$, and $\SymBlockDiag_d(A)$ be the
operator extracting the symmetrization of the $d \times d$ block-diagonal part
of $A$.

\paragraph{Probability and statistics:}

We denote the multivariate Gaussian distribution with mean $\mu \in \R^d$ and
covariance $\Sigma \in \PSD_{+}^d$ as $\Gaussian(\mu, \Sigma)$. We denote the
isotropic Langevin distribution on $\SO(d)$ with mode $M \in \SO(d)$ and
concentration parameter $\kappa \geq 0$ as $\Langevin(M, \kappa)$; this is the
distribution whose probability density function is:
\begin{equation}
  p(R; M, \kappa) = \frac{1}{c_d(\kappa)} \exp \left( \kappa \tr \left( M^TR \right) \right),
\end{equation}
with respect to the Haar measure on $\SO(d)$, with $c_d(\kappa)$ a normalization
constant.

Finally, for an unknown variable $Z$ we aim to infer, we denote its true latent
value by $\true{Z}$ and a noisy measurement of $Z$ by $\noisy{Z}$.

\paragraph{Gauge-invariant distance metrics:}

A key property of rotation averaging and pose-graph optimization is that, as
synchronization problems, they admit infinitely many solutions due to
\emph{gauge symmetry}. In particular, we will see that if $\Ropt \in \SO(d)^n$
is an optimal estimate of the rotational states, then $G \Ropt$ is also optimal
for any $G \in \SO(d)$. We therefore define the following \emph{orbit distances}
in order to compare solutions to the problems in a symmetry-aware manner:
\begin{subequations}
  \begin{equation}
    \Sorbdist(X, Y) \triangleq \min_{G \in \SO(d)} \| X - GY \|_F, \quad X,Y \in \SO(d)^n
  \end{equation}
  \begin{equation}
    \Oorbdist(X, Y) \triangleq \min_{G \in \Orthogonal(d)} \| X - GY \|_F, \quad X,Y \in \Orthogonal(d)^n.
  \end{equation}
\end{subequations}
It will be convenient to ``overload'' the $\Orthogonal(d)$ orbit distance to act
on elements of the set $\eigvecSpace \triangleq \left\{ Y \in \R^{d \times dn}
  \mid YY\transpose = nI_d \right\}$.\footnote{The elements of $\eigvecSpace$
  admit a straightforward interpretation as transposed and re-scaled elements of
  the Stiefel manifold $\Stiefel(d, dn)$ (see \eqref{eq:stiefel-def}).} That is,
for $X, Y \in \eigvecSpace$:
\begin{equation}
  \Oorbdist(X, Y) \triangleq \min_{G \in \Orthogonal(d)} \| X - GY \|_F.
\end{equation}
Each of these distances can be computed in closed form by means of a singular value decomposition (see \citet[Theorem 5]{rosen2019se}).

\subsection{Problem formulation}

We consider the problem of synchronization over the $\SO(d)$ group: this is the
problem of estimating $n$ unknown values $\rot_1, \ldots, \rot_n \in \SO(d)$
given a set of noisy measurements $\nrot_{ij}$ of a subset of their pairwise
relative rotations $\trot_{ij} \triangleq \trot_i\inv \trot_j$. The problem of $\SO(d)$-synchronization captures, in
particular, the problems of rotation averaging and, under common modeling
assumptions, pose graph optimization (as we show in Problem \ref{prob:rot-sync}
and equation \eqref{eq:nq-pgo}), where the variables of interest are the
orientations of a robot (or more generally, a rigid body) at different points in
time (see, for example \citet{grisetti2010tutorial}). This problem possesses a
natural graphical structure $\Graph \triangleq (\Nodes, \dEdges)$, where nodes
$\Nodes$ correspond to latent variables $\rot_i \in \SO(d)$ and edges $(i,j) \in
\dEdges$ correspond to (noisy) measured relative rotations $\nrot_{ij}$ between
$\rot_i$ and $\rot_j$. In particular, for the problem of \emph{rotation
  averaging}, we adopt the following standard generative model for rotation
measurements: For each edge $\dedge \in \dEdges$, we sample a noisy relative measurement $\nrot_{ij}$ according to (cf.\ \cite{rosen2019se,
  dellaert2020shonan}):
\begin{equation}\label{eq:gen-model-ra}
  \nrot_{ij} = \trot_{ij}\rot_{ij}^\epsilon, \quad \rot_{ij}^{\epsilon} \sim \Langevin(I_d, \kappa_{ij}).
\end{equation}
Given a set of noisy pairwise relative rotations $\nrot_{ij}$ sampled according to the generative model
\eqref{eq:gen-model-ra}, a maximum likelihood estimate $\Ropt \in \SO(d)^n$ for
the latent rotational states $\rot_1, \ldots, \rot_n$ is obtained as a minimizer
of the following problem \cite{rosen2019se, dellaert2020shonan}:
\begin{problem}[Maximum likelihood estimation for rotation averaging]
  \label{ra-mle}
  \begin{equation}
    \min_{\rot_i \in \SO(d)} \sum_{(i,j) \in \dEdges} \kappa_{ij} \| \rot_j - \rot_i\nrot_{ij}\|^2_F.
  \end{equation}
\end{problem}

For pose-graph SLAM ($\SE(d)$-synchronization), we adopt the following
generative model for rotation and translation measurements: For each edge
$\dedge \in \dEdges$, we sample a noisy relative measurement $\npose_{ij} = (\ntran_{ij}, \nrot_{ij}) \in \SE(d)$ according to:
\begin{subequations}
  \begin{equation}
    \nrot_{ij} = \trot_{ij}\rot_{ij}^\epsilon, \quad \rot_{ij}^{\epsilon} \sim \Langevin(I_d, \kappa_{ij})
  \end{equation}
  \begin{equation}
    \ntran_{ij} = \ttran_{ij} + \tran_{ij}^\epsilon, \quad \tran_{ij}^{\epsilon} \sim \Gaussian(0, \transym_{ij}^{-1}I_d),
  \end{equation}
  \label{eq:gen-model-pgo}%
\end{subequations}
where $\tpose_{ij} =  \tpose_i^{-1}\tpose_j = (\ttran_{ij}, \trot_{ij})$ is the true relative transformation from $x_i$ to $x_j$. Under this noise model, a maximum likelihood estimate $x^{\optsym} \in \SE(d)^n$ for the latent states $x_1, \ldots, x_n$ is obtained as a minimizer
of the following problem \cite{rosen2019se}:
\begin{problem}[Maximum likelihood estimation for $\SE(d)$ synchronization]
  \label{se-mle}
  \begin{equation}
    \min_{\substack{{\tran_i \in \R^d} \\ {\rot_i \in \SO(d)}}} \sum_{(i,j) \in \dEdges} \kappa_{ij} \| \rot_j - \rot_i\nrot_{ij}\|^2_F + \transym_{ij}\|t_j - t_i - R_i\ntran_{ij}\|_2^2.
  \end{equation}%
\end{problem}

Note that under these modeling assumptions, both pose-graph optimization and rotation
averaging can be written as particular instances of the following general optimization problem:
\begin{problem}[Quadratic minimization over $\SO(d)^n$]\label{prob:rot-sync}
  \begin{equation}
    p^* = \min_{\rot \in \SO(d)^n} \tr(\nQ \rot\transpose \rot),
  \end{equation}
  where $\nQ \in \Sym(dn),\ \nQ \succeq 0$.
\end{problem}%
\noindent Specifically, the problems of rotation averaging (RA) and pose-graph
optimization (PGO) in Problems \ref{ra-mle} and \ref{se-mle}, respectively, can
be parameterized in terms of the following data matrices:
    \begin{subequations}
      \begin{equation}\label{eq:nq-ra}
        \nQ = \MeasRotConLap, \tag{RA}
      \end{equation}
      \begin{equation}\label{eq:nq-pgo}
        \nQ = \MeasRotConLap + \nQtran, \tag{PGO}
      \end{equation}
  \end{subequations}
  where $\MeasRotConLap$ is the \emph{rotation connection Laplacian} and
  $\nQtran$ is a data matrix comprised of translation measurements. For the
  purposes of this paper, the specific structure of $\nQ$ is not important; we
  require only that in the \emph{noiseless case}, where $\nQ = \tQ$, we have
  $\trot\transpose \in \ker(\tQ)$, where $\trot$ is the set of (latent)
  ground-truth rotational states, and  $\TrueRotConLap \succeq 0$
  and $\tQtran \succeq 0$ (see \cite[Appendix C.3]{rosen2019se} for a detailed
  analysis of the noiseless case). Finally, the interested reader may refer to
  Appendix \ref{app:data-matrix} for a complete description of these data
  matrices.

\section{Spectral methods for initialization}\label{sec:spectral-init}

\begin{small}
  \begin{algorithm}[t]
    \caption{Spectral initialization procedure\label{alg:spectral}}
    \begin{algorithmic}[1]
      \Input The data matrix $\nQ$ from \eqref{eq:nq-ra} or \eqref{eq:nq-pgo}
      \Output A spectral initialization $\Rinit$
      \Function{SpectralInitialization}{$\nQ$}
      \State Compute orthogonal set of eigenvectors $\eigvecs$ corresponding to the $d$
      smallest eigenvalues 
      \Statex[1] of $\nQ$.  \Comment{Solve Problem \ref{prob:spectral-relaxation}}.
      \For {$i = 1, \dotsc, n$}
      \State Set $\Rinit_i \leftarrow \SOrounded{\eigvecs_i}$, where $\eigvecs_i$ is the
      $i$-th $(d \times d)$ block of $\eigvecs$. \Comment{Definition \ref{def:so-proj}}
      \EndFor
      \State \Return $\Rinit$
      \EndFunction
    \end{algorithmic}
  \end{algorithm}
\end{small}

The nonconvexity of the $\SO(d)$ constraint renders Problem \ref{prob:rot-sync}
computationally hard to solve in general. However, we can generate a tractable
\emph{spectral relaxation} of Problem \ref{prob:rot-sync} by relaxing the
$\SO(d)$ constraint as follows:
\begin{problem}[Spectral Relaxation of Problem \ref{prob:rot-sync}]\label{prob:spectral-relaxation}
  \begin{equation}
  \label{eq:spectral-relaxation}
    \begin{aligned}
      p^*_{\mathrm{S}} &= \min_{Y \in \R^{d \times dn}} \tr(\nQ Y\transpose Y) \\
      &\mathrm{s.t.}\ YY\transpose = n I_d.
    \end{aligned}
  \end{equation}
\end{problem}%
\noindent Here, the $\SO(d)$ constraint on each $(d \times d)$ block of the
variable $Y$ has been replaced by the (weaker) constraint that $YY\transpose =
nI_d$, i.e. the matrix $Y$ is comprised of $d$ orthogonal rows of norm $\sqrt{n}$.  While the relaxed constraints in \eqref{eq:spectral-relaxation} are still quadratic and nonconvex, in Appendix \ref{app:spectral-relaxation-analysis} we prove that a feasible point $Y$ is
a (\emph{global}) minimizer of Problem \ref{prob:spectral-relaxation} if and only if its rows are comprised of
  $d$ pairwise orthogonal (and appropriately scaled) eigenvectors corresponding to the minimum $d$
eigenvalues of $\nQ$.
Therefore, one can recover an optimizer $Y^{\optsym}$ of Problem
\ref{prob:spectral-relaxation} via a simple eigenvector computation.\footnote{This justifies our referring to Problem \ref{prob:spectral-relaxation} as a ``spectral'' relaxation of Problem \ref{prob:rot-sync}.}

For the noiseless problem parameterized by $\tQ$, the relaxation in Problem
\ref{prob:spectral-relaxation} is exact in the sense that $\trot = GY^{\optsym}$
for some $G \in \Orthogonal(d)$.\footnote{The spectral relaxation in Problem
  \ref{prob:spectral-relaxation}, like Problem \ref{prob:rot-sync}, admits
  infinitely many solutions: if $Y^{\optsym}$ is a minimizer of Problem
  \ref{prob:spectral-relaxation}, then any $GY^{\optsym}, G \in \Orthogonal(d)$
  is also a minimizer.} This follows from the fact that, by construction, the
ground truth rotations $\trot\transpose$ lie in $\ker(\tQ)$,\footnote{We refer
  the reader to \cite[Appendix C.3]{rosen2019se} for detailed analysis of the
  noiseless case.} and $\trot \trot\transpose = nI_d$ since $\trot \in
\SO(d)^n$. Likewise, since $\trot$ is a minimizer of the relaxed problem
\emph{and} is in the feasible set for the Problem \ref{prob:rot-sync}, it is
also a minimizer for Problem \ref{prob:rot-sync}. In general, however, we do not expect such a nice correspondence to hold.  Indeed, 
a minimizer of Problem \ref{prob:spectral-relaxation} need not even be \emph{feasible} for
 Problem \ref{prob:rot-sync}, since the former is obtained from the latter by relaxing constraints. Therefore, we must in general \emph{round} the estimate provided by the spectral relaxation to obtain an approximate solution $\Rinit \in \SO(d)^n$ in the feasible set of Problem
\ref{prob:rot-sync}. The following definition makes this precise.
\begin{definition}[Projection onto $\SO(d)$]\label{def:so-proj}
  For $X \in \R^{d \times d}$, the projection $\SOrounded{X}$ of $X$ onto
  $\SO(d)$ is by definition a minimizer of the following:
  \begin{equation}
     \min_{G \in \SO(d)} \|X - G\|_F.
  \end{equation}
  A minimizer for this problem is given in closed-form as 
  \cite{hanson1981analysis, umeyama1991least}:
  \begin{equation}
    \SOrounded{X} = U \Xi V\transpose.
  \end{equation}
  where $X = U \Sigma V\transpose$ is a singular value decomposition, and
  $\Xi$ is the matrix:
  \begin{equation}
    \Xi = \Diag\left(1, 1, \det(UV\transpose)\right).
  \end{equation}
  In the context of subsequent derivations, it will be convenient to
  ``overload'' this rounding operation to $Y \in \R^{d \times dn}$ as follows:
  \begin{equation}
    \SOrounded{Y} = \left( \SOrounded{Y_1}, \ldots, \SOrounded{Y_n} \right),
  \end{equation}
  where $Y_i \in \R^{d\times d}$ are the $n$ blocks of $Y$.
\end{definition}
Therefore, we can obtain an approximate solution to Problem \ref{prob:rot-sync}
from a minimizer $Y^\optsym$ of the relaxation in Problem
\ref{prob:spectral-relaxation} as $\Rinit \triangleq \SOrounded{Y^\optsym}$.
Our overall spectral initialization procedure is summarized in Algorithm \ref{alg:spectral}.

\section{Main results}\label{sec:main-results}

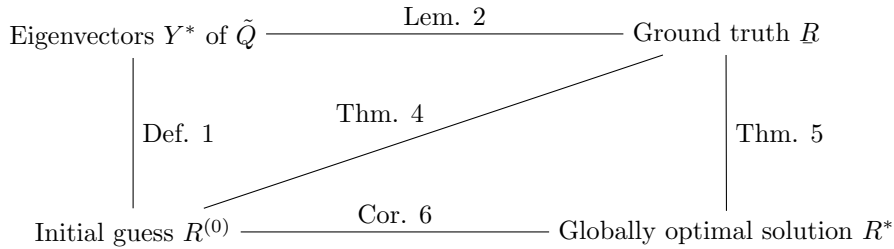
\begin{figure}[t]
  \centering
  \begin{tikzpicture} [align=center,xscale=4]
    \node (evecs) at (0,0) {Eigenvectors $\eigvecs$ of $\nQ$};
    \node (rinit) [below=2cm of evecs] {Initial guess $\Rinit$};
    \node (rtrue) [right=4.75cm of evecs] {Ground truth $\trot$};
    \node (ropt) [right=4.1cm of rinit] {Globally optimal solution $\Ropt$};
    \draw[-]
    (evecs) edge node[auto] {Lem. \ref{lem:alignment-eigvecs-trot}} (rtrue)
    (rinit) edge node[auto] {Thm. \ref{thm:spectral-error}} (rtrue)
    (rtrue) edge node[auto] {Thm. \ref{thm:opt-error}} (ropt)
    (rinit) edge node[auto] {Cor. \ref{cor:rinit-vs-ropt}} (ropt);
    \draw[-]
    (evecs) edge node[auto] {Def. \ref{def:so-proj}} (rinit);
  \end{tikzpicture}
  \caption{\textbf{Guide to the main results.} This figure presents a
    diagrammatic guide to the bounds presented in Section
    \ref{sec:main-results}. Results here are represented by the edges between
    the quantities they relate. In particular, Lemma
    \ref{lem:alignment-eigvecs-trot} gives a bound on the orbit distance between
    the eigenvectors of the data matrix and the ground truth. We then use this
    result in Theorem \ref{thm:spectral-error}, giving a bound on the deviation
    of the spectral initialization from the ground truth. In Theorem
    \ref{thm:opt-error} we bound the deviation of a \emph{globally optimal
      solution} to the maximum likelihood estimation problems for rotation
    averaging and pose-graph SLAM from the ground truth. Finally, relating these
    bounds we obtain Corollary \ref{cor:rinit-vs-ropt}, which bounds the
    distance between a spectral initialization and the globally optimal
    solution. \label{fig:guide}}
\end{figure}

This section presents our main results, which are three-fold: First, we provide
a bound on the error of our spectral initialization $\Rinit$ with respect to the
ground-truth rotations $\trot$. Second, we give a new bound on the error of
\emph{globally optimal} solutions $\Ropt$ with respect to $\trot$: this bound
differs from prior work (e.g. \citet{rosen2019se, preskitt2018phase}) in that it
is defined with respect to the orbit distance $\Sorbdist$ on $\SO(d)^n$.
Previous work used the orbit distance $\Oorbdist$ on $\Orthogonal(d)^n$ due to
mathematical convenience; however, the estimation error one considers in
application is actually over $\SO(d)^n$, since this is the domain on which the
estimation problem is defined. Combining these results, we obtain an upper bound
on the $\SO(d)$ orbit distance between an initial guess $\Rinit$ and a globally
optimal solution $\Ropt$. Our analysis gives direct control over the mutual
deviation between the three quantities of interest: $\Rinit$, $\Ropt$, and
$\trot$ as a function of the noise magnitude. We conclude with additional
remarks about computing these bounds for practical SLAM scenarios and a few
straightforward adaptations of the main results. Figure \ref{fig:guide} gives an
overview of the main results.

Recall from Problem \ref{prob:spectral-relaxation} that an estimate $\eigvecs$
is a minimizer of Problem \ref{prob:spectral-relaxation} if and only if it is
composed of a (suitably scaled) orthogonal set of eigenvectors corresponding to the minimum $d$ eigenvalues of $\nQ$, and that in the noiseless case a minimizer is given
by $\trot$. Since a spectral initialization $\Rinit$ is obtained as the
projection of a solution $\eigvecs$ of Problem \ref{prob:spectral-relaxation}
onto $\SO(d)^n$, we can bound its estimation error by first bounding the
deviation of $\eigvecs$ from $\trot$, then bounding the additional error
incurred by projecting onto $\SO(d)^n$.

We will begin our presentation of the main results by giving a bound on the
deviation of a solution $\eigvecs$ of Problem \ref{prob:spectral-relaxation}
from the ground truth $\trot$ via the Davis-Kahan Theorem \cite{yu2015useful}, a
classical result relating the perturbation of a matrix's eigenvectors under a
symmetric perturbation to the magnitude of that perturbation. Here, we take
$\tQ$ to be the matrix under consideration, and define the perturbation $\dQ
\triangleq \nQ - \tQ$. The following lemma, which we prove in Appendix
\ref{app:eigen-perturbation}, gives the desired characterization:
\begin{lem} \label{lem:alignment-eigvecs-trot} Let $\eigvecs$ be a minimizer of
  Problem \ref{prob:spectral-relaxation} and $\trot$ be the corresponding ground
  truth rotations. Then:
  \begin{equation}
    \Oorbdist\left(\trot, \eigvecs\right) \leq \frac{2\sqrt{2dn}\|\dQ\|_2}{\lambda_{d+1}(\tQ)}. \label{eq:eigvec-bound}
  \end{equation}
\end{lem}
Lemma \ref{lem:alignment-eigvecs-trot} provides control over the deviation of an
``unrounded'' solution $\eigvecs$ from the ground truth $\trot$. The second technical
ingredient we require is the following simple bound controlling the maximum
distance between a matrix $X$ and its projection $\SOrounded{X}$ onto $\SO(d)$:
\begin{lem}\label{lem:so-to-o-rounding}
  Let $X \in \R^{d \times d}$ and $R \in \SO(d)$. Then:
  \begin{equation}
    \|\SOrounded{X} - R \|_F \leq 2\|X - R\|_F.
  \end{equation}
\end{lem}
\begin{proof}
  \begin{align}
    \|\SOrounded{X} - R \|_F &= \|\SOrounded{X} - X + X - R\|_F \\
                             &\leq \|\SOrounded{X} - X\|_F + \|X - R\|_F \\
                             &\leq 2 \|X - R\|_F,
  \end{align}
  where the last inequality follows from the fact that $\SOrounded{X}$ is a
  minimizer over $\SO(d)$ of the distance to $X$ with respect to the Frobenius
  norm, and that, by hypothesis, $R \in \SO(d)$.
\end{proof}
Lemma \ref{lem:so-to-o-rounding} provides a straightforward approach for
converting a bound expressed in the $\Orthogonal(d)^n$ orbit distance to one
expressed in the $\SO(d)^n$ orbit distance. In turn, we obtain the following
theorem, which we prove in Appendix \ref{app:proof-spectral-error}:
\begin{thm}\label{thm:spectral-error}
  Let $\eigvecs$ be a minimizer of Problem \ref{prob:spectral-relaxation} and
  $\Rinit = \SOrounded{\eigvecs}\in \SO(d)^n$ be the corresponding spectral
  initialization. Finally, let $\trot \in \SO(d)^n$ be the set of ground truth
  rotations in Problem \ref{prob:rot-sync}. Then the estimation error of
  $\Rinit$ satisfies:
  \begin{equation}
    \Sorbdist(\trot, \Rinit) \leq \frac{4\sqrt{2dn}\|\dQ\|_2}{\lambda_{d+1}(\tQ)}. \label{eq:spectral-error}
  \end{equation}
\end{thm}
The bound \eqref{eq:spectral-error} gives a direct (linear) relationship between
the magnitude of the perturbation $\dQ$ and the worst-case error of a spectral
estimate. Moreover, Theorem \ref{thm:spectral-error} implies that
$\Sorbdist(\trot, \Rinit) \rightarrow 0$ as $\dQ \rightarrow 0$. That is to say,
as the measurements approach their noiseless counterparts, our spectral estimate
approaches the ground truth.

Next, we address the issue of furnishing a bound on $\Sorbdist(\trot, \Ropt)$.
The following theorem, which we prove in Appendix \ref{app:proof-mle-bound},
gives the desired result:
\begin{thm}[Bounding the estimation error for $\Ropt$]\label{thm:opt-error}
  Let $\Ropt$ be a minimizer of Problem \ref{prob:rot-sync} and $\trot$ be the
  set of ground-truth rotations. Then the estimation error of $\Ropt$ satisfies:
  \begin{equation}
    \Sorbdist(\trot, \Ropt) \leq \frac{8\sqrt{dn}\|\dQ\|_2}{\lambda_{d+1}(\tQ)}. \label{eq:opt-error}
  \end{equation}
\end{thm}
To the best of our knowledge, Theorem \ref{thm:opt-error} is the first result to
appear in the literature that directly controls the estimation error of the
maximum likelihood estimate $\Ropt$ over $\SO(d)^n$ specifically. Prior work
considered the estimation error over $\Orthogonal(d)^n$ \cite{ling2020near,
  bandeira2017tightness, rosen2019se}. In our application, however, we are
specifically concerned with the estimation error over $\SO(d)^n$; as one can see
from inspection, this is the domain on which Problem \ref{prob:rot-sync} is
defined. Thus, the $\SO(d)^n$ orbit distance corresponds to the actual error one
would obtain in practice.

While Theorem \ref{thm:spectral-error} establishes error bounds for the spectral
estimator, when viewed as an \emph{initialization method}, the distance between
the initial guess $\Rinit$ and the globally optimal solution is the primary
concern. A corollary to Theorems \ref{thm:spectral-error} and
\ref{thm:opt-error}, allows us to control $\Sorbdist(\Rinit, \Ropt)$ in terms of
the noise matrix $\dQ$. We have:
\begin{cor}\label{cor:rinit-vs-ropt}
  The orbit distance between the initialization $\Rinit$ and a globally optimal
  solution $\Ropt$ satisfies:
  \begin{equation}
    \Sorbdist(\Rinit, \Ropt) \leq \frac{(8 + 4 \sqrt{2})\sqrt{dn}\|\dQ\|_2}{\lambda_{d+1}(\tQ)}. \label{eq:spectral-opt-dev}
  \end{equation}
\end{cor}

These bounds provide a clear relationship between the spectral properties of
$\tQ$ and $\dQ$ and the deviation between a spectral estimator $\Rinit$,
maximum likelihood estimator $\Ropt$, and the ground-truth $\trot$. An important
consequence of these bounds is that as $\dQ \rightarrow 0$, we have (at least)
linear convergence of the estimation error for \emph{both} the spectral
estimator and the maximum likelihood estimator to zero. This, in turn,
guarantees that $\dQ \rightarrow 0$ implies $\Ropt, \Rinit \rightarrow \trot$
(up to symmetry), which is what we would expect.

In practice, however, we do not have access to $\tQ$. This presents some
difficulty in the computation of $\dQ$ and $\lambda_{d+1}(\tQ)$. Fortunately,
the noiseless rotation matrices admit a description in terms of quantities that
\emph{are} typically assumed to be known. In particular, we have \cite[Lemma
8]{rosen2019se}:
\begin{equation}
  \lambda_{d+1}(\TrueRotConLap) = \lambda_2(\LapRotW),
\end{equation}
where $\LapRotW$ is the Laplacian of the rotational weight graph. Now,
$\LapRotW$ depends only on the concentration parameters $\kappa_{ij}$ attached
to each edge, which are generally assumed to be known \emph{a priori} from the
noise models \eqref{eq:gen-model-ra} and \eqref{eq:gen-model-pgo}. In the
rotation averaging case, we have $\tQ = \TrueRotConLap$, and therefore the
denominator $\lambda_{d+1}(\tQ)$ is readily available as
$\lambda_2(\LapRotW)$, the algebraic connectivity of the rotational weight
Laplacian.

In the case of pose-graph SLAM, where the matrix $\tQ$ contains the
translational terms $\tQtran$, we can use the fact that $\tQ = \TrueRotConLap +
\tQtran$ is the sum of positive-semidefinite matrices (see \citet[Appendix
C.3]{rosen2019se}), so $\lambda_{d+1}(\TrueRotConLap) \leq
\lambda_{d+1}(\TrueRotConLap + \tQtran) = \lambda_{d+1}(\tQ)$. In particular,
the (weaker) bounds obtained by substituting $\lambda_{d+1}(\tQ)$ with
$\lambda_{d+1}(\TrueRotConLap)$ in \eqref{eq:spectral-error} and
\eqref{eq:opt-error} hold.

Moreover, a common SLAM initialization technique is that of \emph{rotation only
  initialization} -- i.e., to compute the initializer $\Rinit$ using \emph{only}
the relative rotation measurements \cite{carlone2015initialization}. This can
have computational advantages in practice since $\MeasRotConLap$ is generally
\emph{sparse}; the same cannot be said for the pose-graph SLAM data matrix
$\nQ$, as it arises via analytic elimination of the translational states, in
which case the resulting data matrix $\nQ$ is formed as a (dense) generalized
Schur complement \cite[Appendix B]{rosen2019se}. Interestingly, for pose-graph
SLAM, a spectral initialization $\Rinit$ computed using the eigenvectors of
$\MeasRotConLap$ (i.e. ignoring $\nQtran$) attains the bound:
\begin{equation}
  \Sorbdist(\trot, \Rinit) \leq \frac{4\sqrt{2dn}\|\RotConLapDev\|_2}{\lambda_{d+1}(\TrueRotConLap)}. \label{eq:ronly-bound}
\end{equation}
This bound holds by the same reasoning as Theorem \ref{thm:spectral-error}, but
with the consideration that $\trot\transpose \in \ker(\TrueRotConLap)$.

As a final consideration, typically we do not have access to $\dQ$ (if we did,
we could recover the true data matrix $\tQ$ as $\nQ - \dQ$). In consequence, we
need a method to estimate the likely magnitude of the noise in a given
application. One way of achieving this is via simulation from the generative
model, given a measurement network and associated measurement
precisions.\footnote{Simulating measurements in the case of pose-graph SLAM
  requires knowledge of the ground-truth translation measurement scale, which is
  typically also unavailable in practice. However, the \emph{rotation-only}
  initialization bound \eqref{eq:ronly-bound} applies in general and depends
  only upon the rotation measurements, which can be simulated to produce an
  empirical distribution over the spectral norm of the perturbation matrix.}
This, in turn, gives a sample set from a \emph{distribution} over the bounds
\eqref{eq:spectral-error}, \eqref{eq:opt-error}, and
\eqref{eq:spectral-opt-dev}.

\begin{figure*}[t]
  \centering
  \includegraphics[width=1.0\columnwidth]{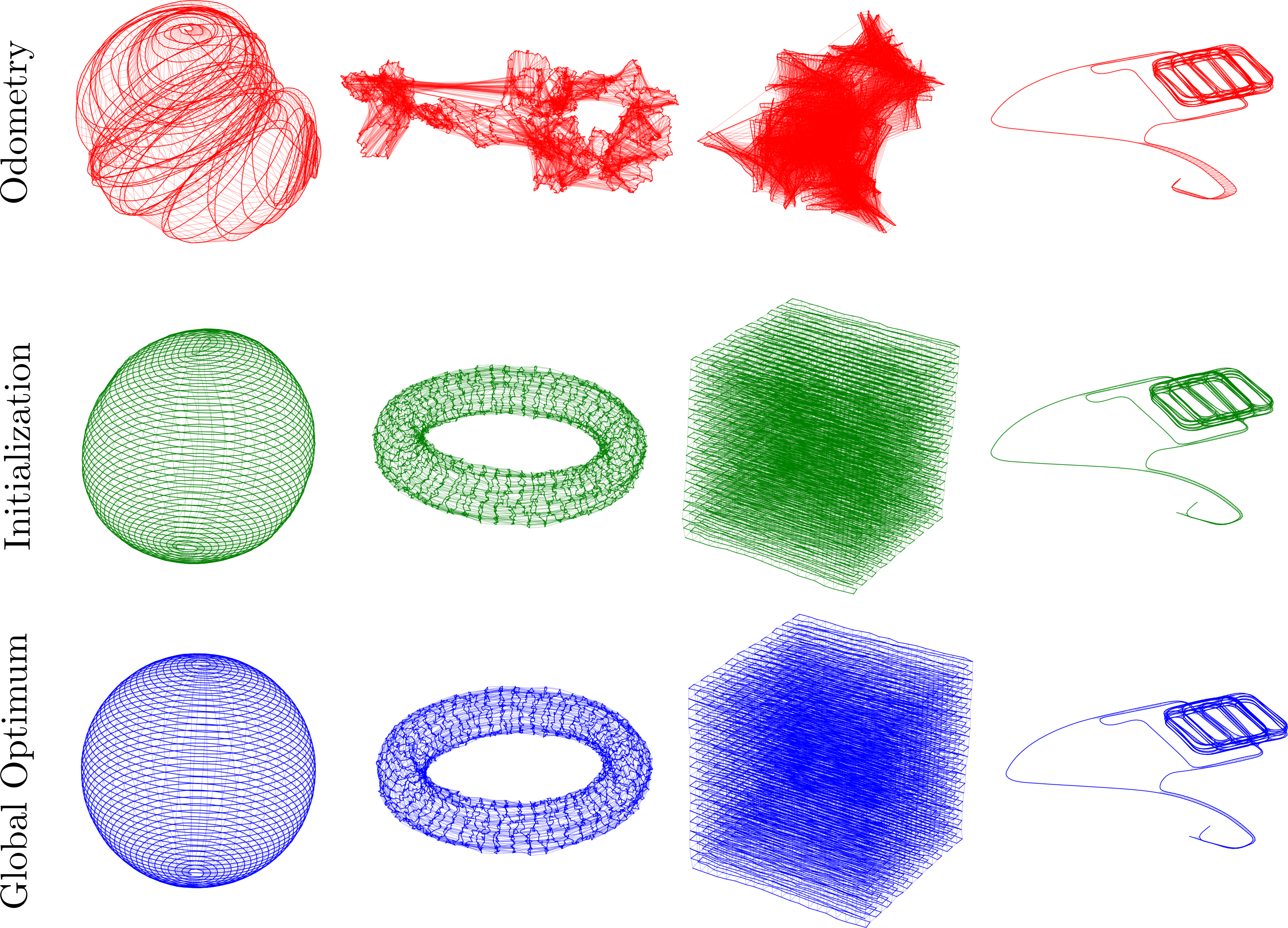}
  \caption{\textbf{Spectral relaxation produces high-quality initializations.}
    Qualitative comparison with the globally optimal solution suggests that the
    spectral relaxation produces estimates that are very close to 
    optimal for a variety of SLAM benchmark
    datasets. The corresponding \emph{quantitative} comparison is given in
    Table \ref{table:stats}. \label{fig:spectral-global}}
\end{figure*}

\section{Experimental results}\label{sec:experimental-results}

In this section, we compare the bounds in Theorem \ref{thm:spectral-error} to
the actual estimation error incurred by the spectral initialization and globally
optimal pose-graph SLAM solutions on a variety of simulated problem instances,
as well as benchmark SLAM problems. In Section \ref{sec:synthetic-data}
we construct synthetic pose-graph SLAM scenarios for which the ground-truth
poses are known. Since the bounds we have presented depend upon knowledge of the
noise magnitude $\|\dQ\|_2$ and the spectral gap of the \emph{true} data matrix
$\tQ$, which are unknown in practice for pose-graph SLAM, our first set of
empirical results shed light on the behavior of these worst-case bounds (as well
as the \emph{actual} error realized by different estimators) as we vary the
noise parameters controlling the generative model \eqref{eq:gen-model-pgo}. In
Section \ref{sec:benchmark-data}, we evaluate the performance of spectral
relaxation as a practical initialization method in the context of 3D pose-graph
SLAM applications. We show that, consistent with our results on synthetic data,
the spectral initialization method offers high-quality initial solutions for
pose-graph optimization, and in particular, that the inclusion of translational
measurements significantly improves the quality of the spectral estimator versus the common approach of using exclusively rotational
measurements.

The spectral initialization method was implemented in C++ using Spectra to
efficiently solve large-scale eigenvalue problems \cite{spectralib}. Computation
of the bounds in Section \ref{sec:synthetic-data} was performed in MATLAB using
\texttt{eigs}. All experiments were performed on a laptop with a 2.2 GHz Intel
i7 CPU. Where (verified) globally optimal solutions were needed, we used the C++
implementation of SE-Sync \cite{rosen2019se}. We also provide results using the
well-known \emph{chordal} initialization method \cite{martinec2007robust}, which
relaxes the feasible set of Problem \ref{prob:rot-sync} to $\R^{d \times dn}$,
with the constraint that $\Rinit_1 = I_d$, for which the solution can be
obtained by solving a linear system.

\subsection{Evaluation on synthetic data}\label{sec:synthetic-data}

\begin{figure}
  \centering
  \includegraphics[width=0.5\columnwidth]{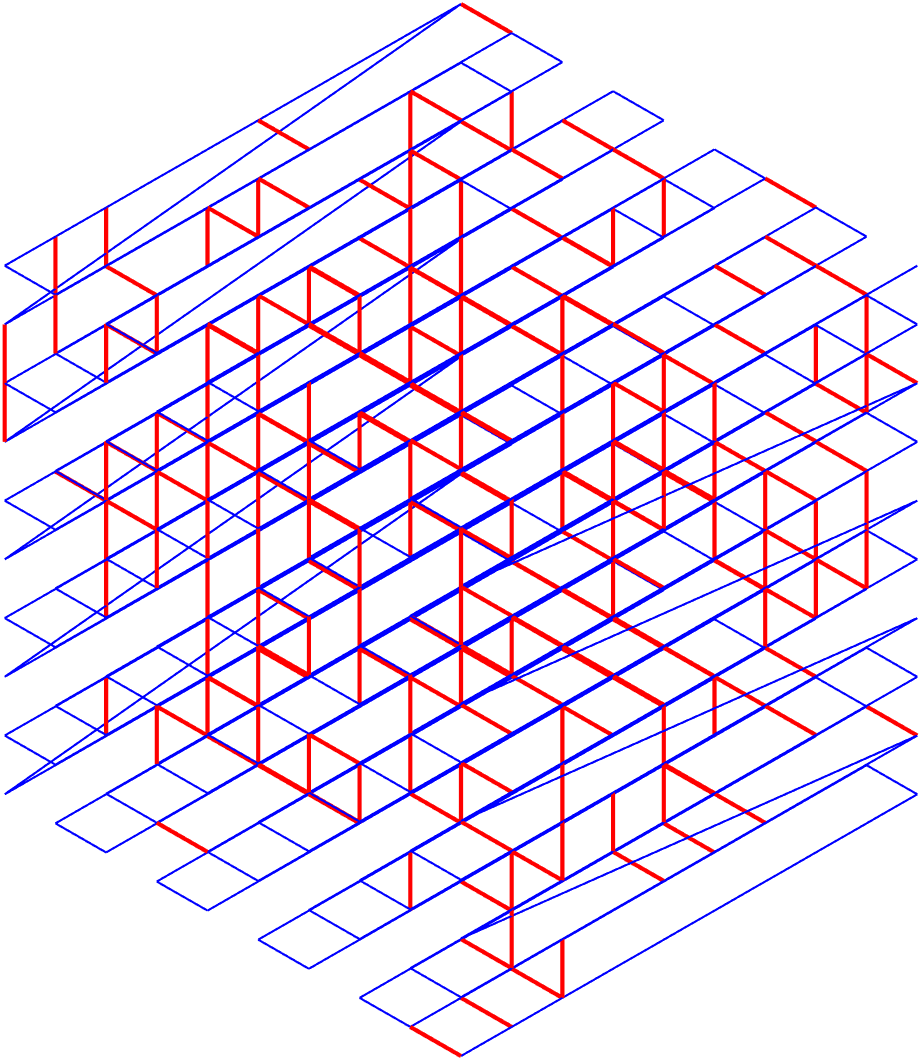}
  \caption{\textbf{\Cube{} experiments.} Example ground-truth realization of a
    synthetic \Cube{} dataset \cite{carlone2015lagrangian, rosen2019se} with
    $s = 10$ vertices per side and $p_{LC} = 0.1$. The robot's trajectory is
    shown in blue with loop closures shown in red.\label{fig:cube-dataset}}
\end{figure}

The bounds presented in our analysis depend upon knowledge of the noise
magnitude $\|\dQ\|_2$, which is unknown in practice. In light of this fact, we
examine empirically the behavior of the bounds as a function of the noise
parameters using synthetic data. Specifically, we use the \Cube{} dataset
\cite{carlone2015lagrangian, rosen2019se}, which consists of a set of vertices
(poses) organized in a three-dimensional cube, with $s$ vertices per dimension.
Consecutive poses have an ``odometry'' edge between them, and loop closures are
sampled randomly from the remaining edges with probability $p_{LC}$.
Measurements are generated by randomly sampling from the generative model
\eqref{eq:gen-model-pgo} with fixed noise parameters $\kappa$ and $\tau$ for all
measurements. Beyond providing access to the ground-truth rotations, this setup
allows us to compare the worst-case bounds with empirical performance in noise
regimes well outside the range typically encountered in real SLAM scenarios. A
sample configuration for the \Cube{} dataset is provided in Figure
\ref{fig:cube-dataset}.

\begin{figure}
  \centering
  \begin{subfigure}{0.5\columnwidth}
    \centering
    \includegraphics[width=1.0\columnwidth]{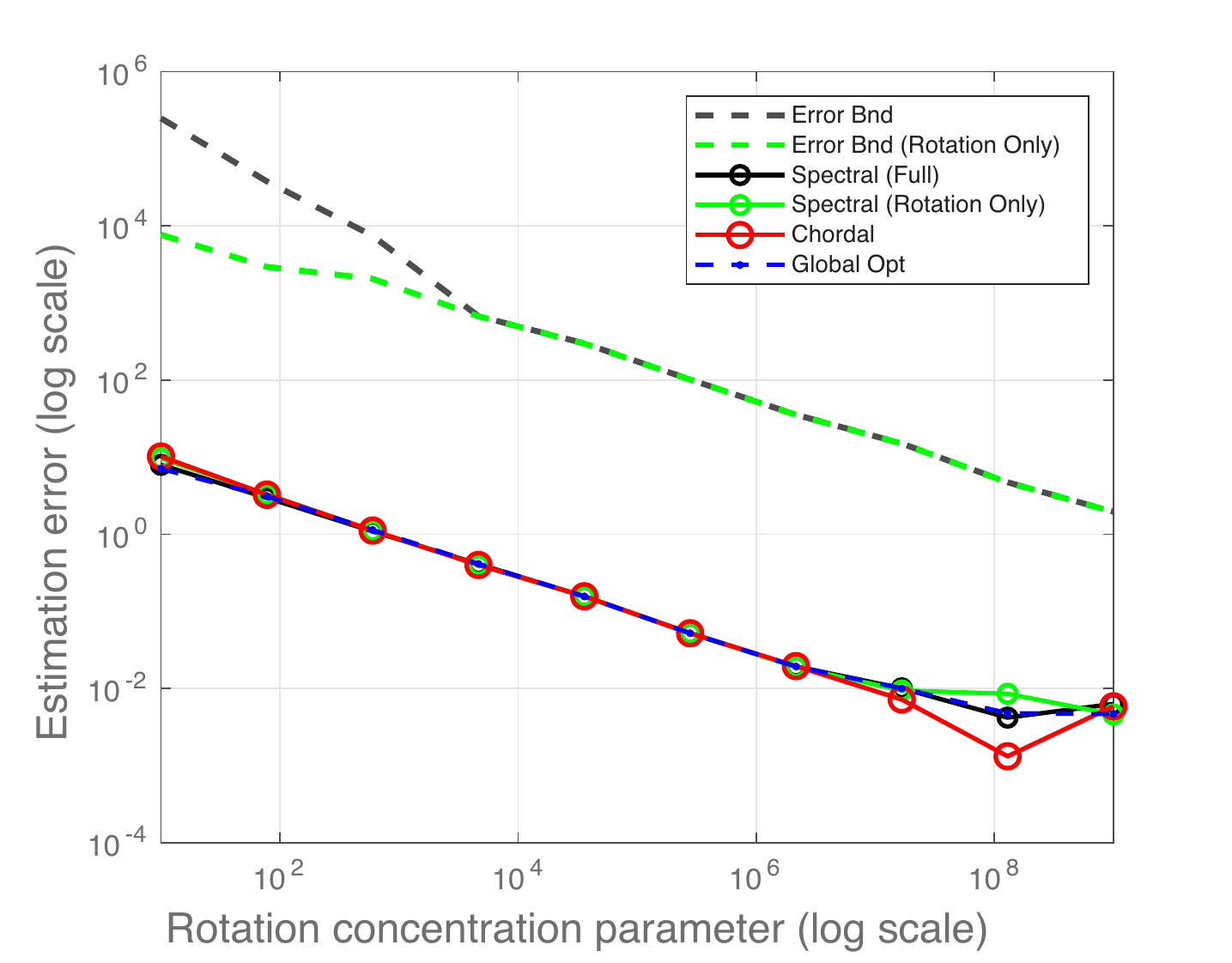}
    \caption{\label{subfig:bound-rotation}}
  \end{subfigure}%
  \begin{subfigure}{0.5\columnwidth}
    \centering
    \includegraphics[width=1.0\columnwidth]{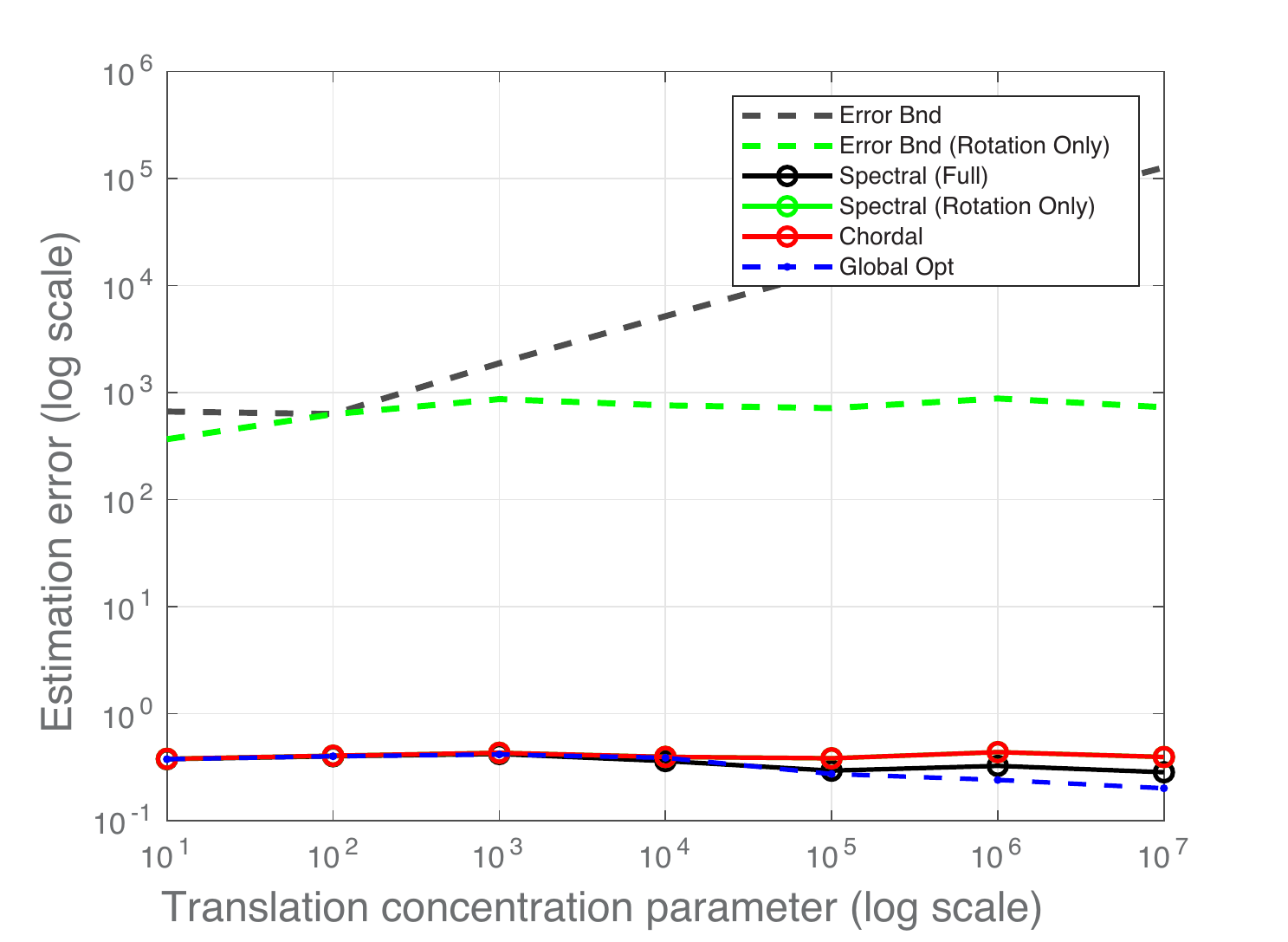}
    \caption{\label{subfig:bound-translation}}
  \end{subfigure}
  \begin{subfigure}{0.5\columnwidth}
    \centering
    \includegraphics[width=1.0\columnwidth]{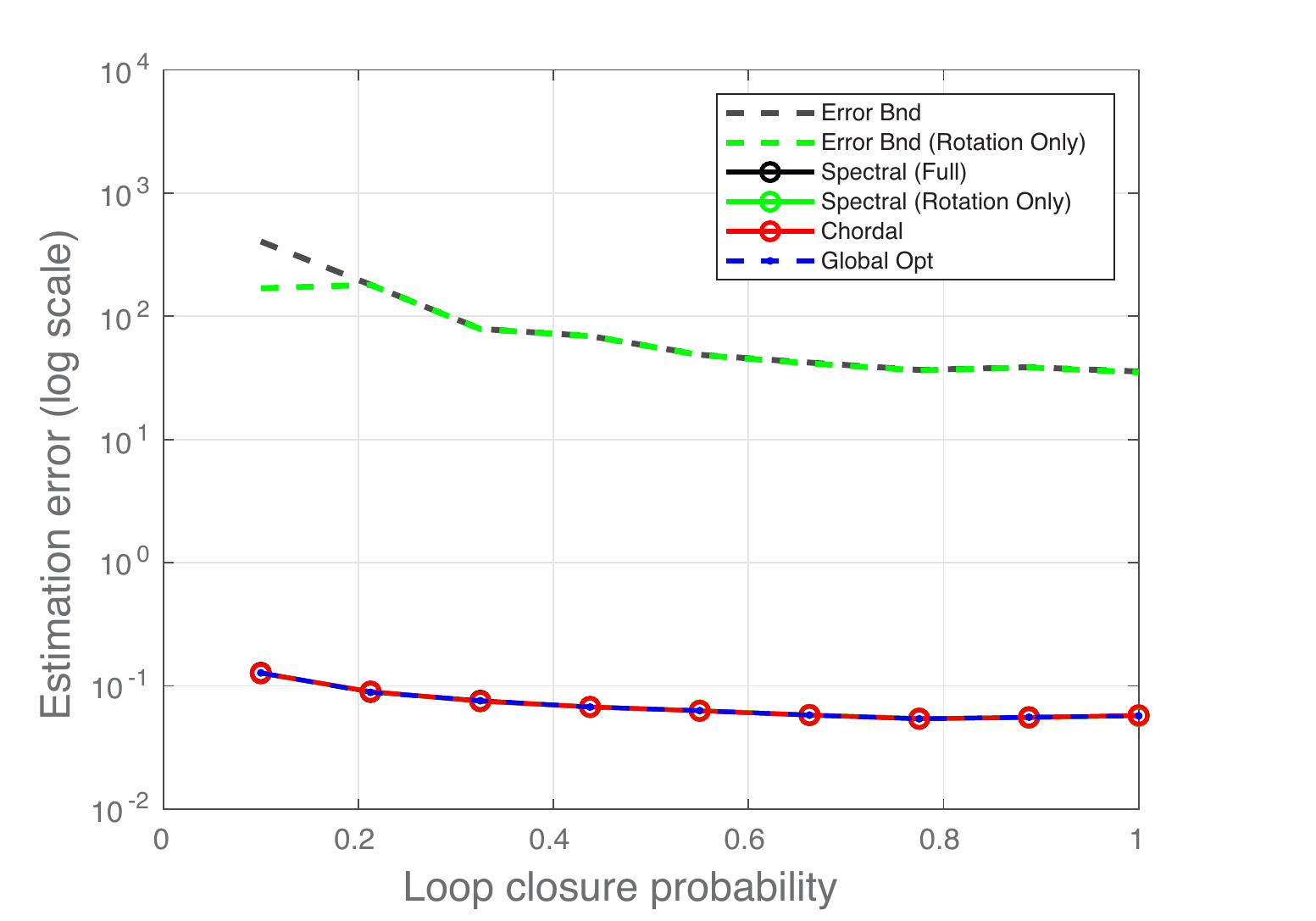}
    \caption{\label{subfig:bound-plc}}
  \end{subfigure}%
  \begin{subfigure}{0.5\columnwidth}
    \centering
    \includegraphics[width=1.0\columnwidth]{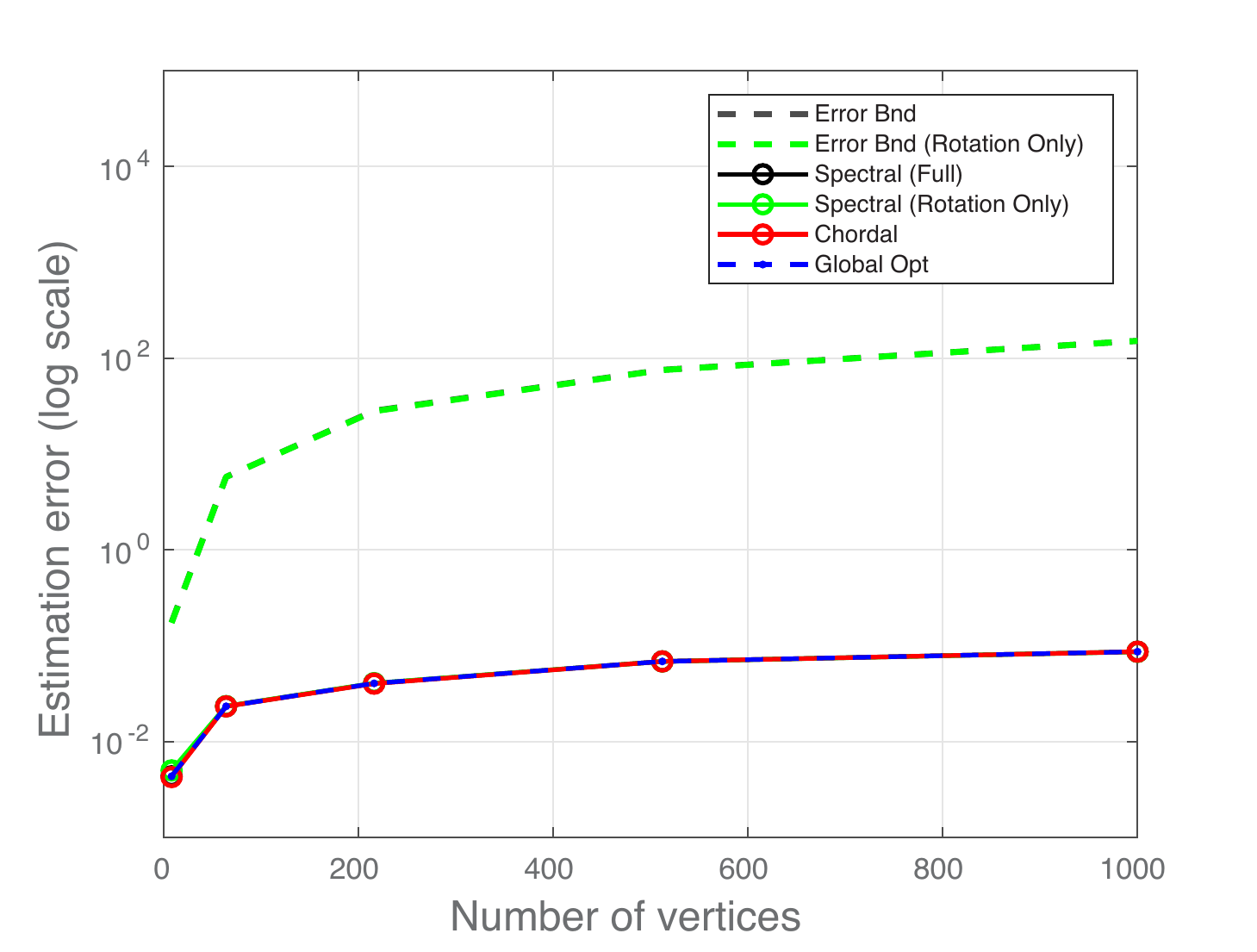}
    \caption{\label{subfig:bound-s}}
  \end{subfigure}
  \caption{\textbf{Influence of dataset parameters on the performance bounds for
      the \Cube{} experiments.} We examine empirically the change in the
    theoretical bounds \eqref{eq:spectral-error} and \eqref{eq:ronly-bound} as
    well as the estimation error of several pose-graph optimization estimates
    while varying (a) the rotation concentration parameter $\kappa$, (b) the
    translation concentration parameter $\tau$, (c) the probability of a loop
    closure $p_{LC}$, (d) the number of vertices $s^3$.\label{fig:cube-experiments}}
\end{figure}

\paragraph{Influence of noise parameters on performance bounds:}

In Figure \ref{fig:cube-experiments}, we study the performance of the spectral
initialization approach across a variety of noise configurations. In each case,
we provide the worst-case bounds \eqref{eq:spectral-error} and
\eqref{eq:ronly-bound} along with the empirical error of the different
estimators under consideration. In Figure \ref{subfig:bound-rotation}, we sample
\Cube{} problem instances with logarithmically spaced values of $\kappa$ while
fixing the other parameters: $\tau = 150$ (corresponding to an expected RMS
error of $0.14$ m), $p_{LC} = 0.2$, and $s = 10$. In Fig.
\ref{subfig:bound-translation}, we fix $\kappa = 10^5$ (corresponding to an
expected RMS error of approximately $0.1^{\circ}$), $p_{LC} = 0.25$ and $s = 10$
and sample problem instances with logarithmically spaced translation
concentration parameter $\tau$. In Fig. \ref{subfig:bound-plc}, we fix $\kappa =
10^5$, $\tau = 150$, $s = 10$ and vary $p_{LC}$ from 0 to 1.

Across a wide range of concentration parameters, the spectral initializations
attain very similar error to the global optimizer.\footnote{$\Ropt$ is the
  maximum likelihood estimator--the optimal point estimate given the data. Since
  there is noise in the data, it is conceivable that the maximum likelihood
  estimate might actually be farther away from the ground truth than a
  ``suboptimal'' estimate, which we observe in Fig.
  \ref{subfig:bound-rotation}.} In particular, their error often improves upon
the worst-case bounds \eqref{eq:spectral-error} and \eqref{eq:ronly-bound} by
orders of magnitude. This is consistent with earlier observations of
qualitatively similar bounds for phase synchronization \cite{preskitt2018phase}.
Moreover, in applications of rotation averaging and pose-graph optimization,
previous work has shown that the maximum likelihood estimator often attains
expected error close to the Cram\'er-Rao \emph{lower bound} (see
\cite{boumal2014cramer} for rotation averaging and \cite{chen2021cramer} for
pose-graph optimization). The behavior of the bounds when varying the
translation concentration parameter in Figure \ref{subfig:bound-translation} is
counterintuitive: while the spectral estimator improves with increasing $\tau$,
the bound suggests the opposite worst-case behavior. It seems the form of the
bounds we derive (including the translational terms) is not refined enough to
capture this behavior, and this certainly warrants further investigation. With
this exception, the bounds seem to accurately capture the behavior of the actual
estimation error.

\paragraph{Dependence on problem dimensionality:}

Due to the explicit appearance of the problem dimension $n$ in the bounds
\eqref{eq:spectral-error}, \eqref{eq:opt-error}, and
\eqref{eq:spectral-opt-dev}, it is interesting to consider how the number of
rotations to be estimated affects these bounds. In Figure \ref{subfig:bound-s},
we fix $\kappa = 10^5$, $\tau = 150$, $p_{LC} = 0.2$ and vary the number of
vertices in the \Cube{} dataset. Indeed, we find that the behavior of the
worst-case bounds suggests an unfavorable scaling in the problem dimension: at
$s^3 = 8$ vertices, the worst-case bound overestimates the true error by
approximately an order of magnitude; at $s^3 = 1000$, it overestimates the true
error by approximately 3 orders of magnitude. It is unclear, at present, whether
it is possible to remove this dependence on the problem dimension. A more
sophisticated analysis considering the specific structure of these matrices (as
defined in Appendix \ref{app:data-matrix}) may yield more refined bounds.

\subsection{Evaluation on standard SLAM benchmark datasets}\label{sec:benchmark-data}

In these experiments, we consider evaluation of the spectral initialization
method on several standard SLAM benchmark datasets. Figure
\ref{fig:spectral-global} provides a qualitative comparison of three techniques
for initialization: odometry only (i.e. composing measurements between
consecutive poses), the proposed spectral initialization approach, and the
globally optimal solution. We observe that spectral initialization provides
solutions that visually resemble the globally optimal solution. Table
\ref{table:stats} gives our quantitative results. For each method, we provide
the computation time, objective value, and number of iterations required for a
Riemannian trust-region (RTR) optimization method to converge to a critical point
when using that initialization. With the exception of
odometry-only initialization, all of the methods considered enabled the recovery
of (verifiably) globally optimal solutions;  that is, these initialization methods coupled with standard \emph{local} optimization techniques recovered
globally optimal solutions \emph{without} the need to explicitly solve a
large-scale semidefinite program.

Both of the spectral methods (using the ``full'' pose-graph optimization data
matrix $\nQ$ and the ``rotation only'' version using only $\MeasRotConLap$)
provide estimates competitive with the state-of-the-art chordal initialization
method, generally attaining near-optimal objective values.\footnote{Our
  current implementation is aimed at recovering high-precision eigenvector
  estimates, rather than expedient computation. Despite this, spectral
  initialization is often faster than the chordal approach, though occasionally
  this added precision leads to longer computation times than would be necessary
  to obtain a good estimate, e.g. on the \Garage{} dataset.} Interestingly, in
their work, \citet{moreira2021fast} found that the rotation-only spectral
estimator attains a higher cost on the \Sphere{} dataset than alternative
methods, as we do here; however, when we include the translation measurements,
we find that this discrepancy disappears. Similarly, the chordal estimator also
performs well on this dataset, despite the fact that, like the rotation-only
spectral initialization, it does not make use of translational measurements.

\begin{table}[t]
  \centering
  \setlength{\tabcolsep}{4pt}
  \begin{tabular}{*{7}{c}}
    \hline
    Dataset & & Odometry & Chordal & \textbf{Spectral (Rotation Only)} & \textbf{Spectral} & Global Opt. \\
    \hline
    \multirow{3}{*}{\Sphere{}}
            & Iter & 65 & 6 & 8 & 4 & \\
            & Cost & 1.14 $\times 10^9$ & 1971.17 & 5594.19  & 1742.75  & 1687 \\
            & Time (s) & -  & 0.707 & 0.602 & 0.779 & \\
    \hline
    \multirow{3}{*}{\Torus{}}
            & Iter & 32 & 5 & 5 & 4 & \\
            & Cost & 3.87 $\times 10^8$ & 24669.2 & 25833.2 & 24272.7 & 24227 \\
            & Time (s) & - & 1.316 & 1.501 & 1.199 & \\
    \hline
    \multirow{3}{*}{\Grid{}}
            & Iter & 30 & 6 & 6 & 4 & \\
            & Cost  & 1.97 $\times 10^{10}$ & 87252 & 86966.1 & 84486.4 & 84320 \\
            & Time (s) & - & 8.747 & 18.806 & 0.25 & \\
    \hline   
    \multirow{3}{*}{\Garage{}}
            & Iter & 1028 & 3 & 4 & 4 & \\
            & Cost & 2.31 $\times 10^9$ & 1.42 & 3.215 & 2.7 & 1.26 \\
            & Time (s) & - & 0.201 & 0.136 & 25.7 & \\
    \hline   
  \end{tabular}
  \caption{\textbf{Standard SLAM benchmarks} Objective value (cost) attained and
    computation time required for each initialization method on several SLAM
    benchmarks. We also report the
    number of iterations (Iter.) required for a Riemannian trust-region optimization method
    to converge to a critical point. Note that the reported computation time is only
    the time required to compute the initialization. Proposed approaches are \textbf{bold}.
  }
  \label{table:stats}
\end{table}

\section{Conclusion}\label{sec:conclusion}

In this work we presented the first initialization methods equipped with
\emph{explicit performance guarantees} adapted to the problems of pose-graph
SLAM and rotation averaging. Our approach is based upon a simple \emph{spectral
  relaxation} of the estimation problem, the form of which permits us to apply
eigenvector perturbation bounds to control the distance from our initialization
to \emph{both} the (latent) ground-truth \emph{and} the global minimizer of the
estimation problem (the \emph{maximum likelihood} estimate) as a function of the
measurement noise. Consistent with recent complementary work on
information-theoretic aspects
\cite{boumal2014cramer,chen2021cramer,khosoussi2014novel} and global
optimization methods \cite{rosen2019se,fan2020cpl,dellaert2020shonan} for SLAM
and RA, our bounds highlight the central role that spectral properties of the
measurement network\footnote{Specifically, the smallest nonzero eigenvalue
  $\lambda_{d+1}(\tQ)$, which can be thought of as a generalization of the
  algebraic connectivity of the classical graph Laplacian.} play in controlling
the accuracy of SLAM and RA solutions. Finally, we show experimentally that our
spectral estimator is very effective in practice, producing initializations of
comparable or superior quality at lower computational cost compared to existing
state-of-the-art techniques.



\appendix 

\section{Structure of the data matrices}\label{app:data-matrix}

In this appendix, we provide the definitions of the various matrices appearing
in the parameterization of the rotation averaging and pose-graph SLAM problems.
$\LapTranW$ and $\LapRotW$ denote the Laplacians of the translational weight
graph $\TranW \triangleq (\Nodes, \Edges, \{\transym_{ij}\})$ and rotational
weight graph $\RotW \triangleq (\Nodes, \Edges, \{\kappa_{ij}\})$, respectively,
with \emph{undirected edges} $\edge \in \Edges$. These are $n \times n$ matrices
with $i,j$-entries:
\begin{subequations}
  \begin{equation}
    \LapTranW_{ij} = \begin{cases} \sum_{e \in \incEdges(i)} \transym_e,& i = j, \\
      -\transym_{ij},& \edge \in \Edges, \\
    0,& \edge \notin \Edges, \end{cases}
   \label{eq:laptranw}
 \end{equation}
 \begin{equation}
   \LapRotW_{ij} = \begin{cases} \sum_{e \in \incEdges(i)} \kappa_e,& i = j, \\
     -\kappa_{ij},& \edge \in \Edges, \\
     0,& \edge \notin \Edges. \end{cases}
     \label{eq:laprotw}
   \end{equation}
\end{subequations}
$\MeasRotConLap$ denotes the \emph{connection Laplacian} for the
rotational measurements, which is a $dn \times dn$ symmetric block-diagonal
matrix with $d \times d$ blocks determined by:
\begin{subequations}
  \begin{equation}
    \MeasRotConLap_{ij} \triangleq \begin{cases} d_i^{\rotsym } I_d,& i = j,\\
      -\kappa_{ij} \nrot_{ij},& \edge \in \Edges,\\
      0_{d \times d},& \edge \notin \Edges, \end{cases}\label{eq:lgrho}
  \end{equation}
  \begin{equation}
    d_i^{\rotsym} \triangleq \sum_{e \in \incEdges(i)} \kappa_e, \label{eq:rotdegree}
  \end{equation}
\end{subequations}
where $\incEdges(i)$ denotes the set of edges \emph{incident to} node $i$.
$\nCrossTerms \in \R^{n \times dn}$ denotes the $(1 \times d)$-block-structured
matrix with $\dedge$ block given by:
\begin{equation}
  \nCrossTerms_{ij} \triangleq \begin{cases} \sum_{e \in \outEdges(j)} \transym_e
    \ntran_e\transpose,& i = j, \\
    -\transym_{ji}\ntran_{ji}\transpose,& (j,i) \in \dEdges,\\
    0_{1 \times d},& \text{otherwise}.
  \end{cases}\label{eq:ncrossterms}
\end{equation}
Finally, $\nOuterProducts \in \SBD(d, n)$ denotes the symmetric block-structured diagonal matrix given by:
\begin{equation}
  \begin{aligned}
    \nOuterProducts &\triangleq \Diag(\nOuterProducts_1, \ldots, \nOuterProducts_n) \in \SBD(d, n) \\
    \nOuterProducts_i &\triangleq \sum_{e \in \outEdges(i)} \transym_e\ntran_e\ntran_e\transpose,
  \end{aligned}\label{eq:nouterproducts}
\end{equation}
where $\outEdges(i)$ denotes the set of edges \emph{leaving} node $i$. With
these definitions in hand, the translational data matrix $\nQtran$ can be
defined as:
\begin{equation}
  \nQtran = \nOuterProducts - \nCrossTerms\transpose \LapTranW \pinv \nCrossTerms.
\end{equation}

\section{Analysis of the spectral relaxation}\label{app:spectral-analysis}

\subsection{Recovering minimizers of Problem \ref{prob:spectral-relaxation} as eigenvectors}\label{app:spectral-relaxation-analysis}
In this section we derive a closed-form description of the \emph{global} minimizers $Y^{\optsym}$ of the spectral relaxation Problem \ref{prob:spectral-relaxation}. Specifically, we prove the following theorem:

\begin{thm}[Global minimizers of the spectral relaxation]\label{thm:spectral-relaxation-minimizer}
Let $\lambda_1(\nQ) \le \dotsb \le \lambda_d(\nQ)$ be the $d$ smallest eigenvalues of $\nQ$.  Then $Y^{\optsym} \in \R^{d \times dn}$ is a global minimizer of the spectral relaxation Problem \ref{prob:spectral-relaxation} if and only if
\begin{equation}
\label{eq:global_minimizer_of_spectral_relaxation}
Y^{\optsym} = \sqrt{n} 
\begin{pmatrix}
 v_{\sigma(1)} \\
 \vdots \\
 v_{\sigma(d)}
\end{pmatrix} \in \R^{d \times dn}
\end{equation}
where $v_1, \dotsc, v_d \in \R^{dn}$ are a set of orthonormal eigenvectors corresponding to the $d$ smallest eigenvalues, and $\sigma$ is a permutation.  The corresponding optimal value of Problem \ref{prob:spectral-relaxation} is:
\begin{equation}
    p^*_{{S}} = n \sum_{i=1}^d \lambda_{i}(\nQ). \label{eq:spectral-opt-val}
\end{equation}
\end{thm}
\begin{proof}
Our approach will be to reduce Problem \ref{prob:spectral-relaxation} to an equivalent problem whose critical points are already well-understood.  To that end, let $Z \triangleq n^{-1/2}Y\transpose \in \R^{dn \times d}$, so that $Y = \sqrt{n}Z\transpose$. Substitution
  into Problem \ref{prob:spectral-relaxation} then gives:
  \begin{equation}
  \label{eq:spectral_relaxation_in_terms_of_Z}
    \begin{aligned}
      p^*_{{S}} = &\min_{Z \in \R^{dn \times d}} \tr\left(n \nQ ZZ\transpose\right) \\
      &\text{s.t.}\ Z\transpose Z = I_d.
    \end{aligned}
  \end{equation}
  Observe that $Z\transpose Z = I_d$ if and only if $Z \in \Stiefel(d, dn)$; therefore, we may in turn rewrite \eqref{eq:spectral_relaxation_in_terms_of_Z} as the following \emph{unconstrained} minimization over the Stiefel manifold: 
 \begin{equation}
    p^*_{{S}} = \min_{Z \in \Stiefel(d,dn)} \tr\left(n \nQ ZZ\transpose\right). \label{eq:rayleigh}
 \end{equation}
 Note that we may now recognize \eqref{eq:rayleigh} as the minimization of a generalized Rayleigh quotient over a Stiefel manifold.  This problem has been extensively studied; in particular, \citet[Section 4.8.2]{absil2009optimization} provides an elementary proof that 
 \begin{equation}
\label{eq:Z_as_column_matrix}
Z = (z_1, \dotsc, z_d) \in \R^{dn \times d}
 \end{equation}
is a critical point of \eqref{eq:rayleigh} \emph{if and only if} its columns $\lbrace z_i \rbrace_{i=1}^d \subset \R^{dn}$ are an orthonormal set of eigenvectors for $n\nQ$.  Moreover, substituting \eqref{eq:Z_as_column_matrix} into the objective in \eqref{eq:spectral_relaxation_in_terms_of_Z} and exploiting the fact that $\lbrace z_i \rbrace_{i=1}^d \subset \R^{dn}$ are pairwise mutually-orthogonal eigenvectors, we find that the corresponding objective value is:
   \begin{equation}
   \label{eq:critical_point_objective_value}
    \tr\left( n \nQ Z Z\transpose \right) = n \sum_{i=1}^d \mu_i,
  \end{equation}
  where $\mu_i$ is the eigenvalue corresponding to $z_i$.  Since every critical point of \eqref{eq:rayleigh} is of the form \eqref{eq:Z_as_column_matrix}--\eqref{eq:critical_point_objective_value}, it follows that the \emph{global minimizers} $Z^{\optsym}$ are precisely those critical points whose columns are composed of the eigenvectors $v_1, \dotsc, v_d \in \R^{dn}$ corresponding to the $d$ \emph{smallest} eigenvalues of $\nQ$. Recovering the corresponding optimal $Y^{\optsym}$ from $Z^{\optsym}$ then gives \eqref{eq:global_minimizer_of_spectral_relaxation} and \eqref{eq:spectral-opt-val}.
\end{proof}

\subsection{Symmetric perturbations of symmetric matrices}\label{app:eigen-perturbation}

Recall that $\trot$ and $\eigvecs$ are solutions of the noiseless and noisy
versions of the spectral relaxation in Problem \ref{prob:spectral-relaxation}.
In Appendix \ref{app:spectral-relaxation-analysis} we showed how these can be
directly obtained from the Stiefel manifold elements giving the $d$
minimum eigenvectors for their corresponding data matrices. The Davis-Kahan
Theorem is a classical result in linear algebra that measures the perturbation
of a matrix's eigenvectors under a symmetric perturbation of that matrix
\cite{stewart1990matrix}. Therefore, we make use of this theorem to derive a
bound on the estimation error of a spectral estimator as a function of the noise
in the data matrix. In particular, the proof of Lemma
\ref{lem:alignment-eigvecs-trot} (and consequently Theorem
\ref{thm:spectral-error}) relies on a particular variant of the Davis-Kahan
$\sin \theta$ Theorem \cite[Theorem 2]{yu2015useful}. Here, we briefly restate
the main result of \cite{yu2015useful} and give a proof of Lemma
\ref{lem:alignment-eigvecs-trot}.

\begin{thm}[\citet{yu2015useful}, Theorem 2]\label{thm:useful-dk-bound}
  Let $\Sigma$, $\hat{\Sigma} \in \R^{p \times p}$ be symmetric, with
  eigenvalues $\lambda_{1} \leq \ldots \leq \lambda_p$ and $\hat{\lambda}_1 \leq
  \ldots \leq \hat{\lambda}_p$ respectively. Fix $1 \leq r \leq s \leq p$ and
  assume that $\min(\lambda_{r} - \lambda_{r-1}, \lambda_{s+1} - \lambda_{s}) > 0$,
  where $\lambda_{0} \triangleq -\infty$ and $\lambda_{p+1} \triangleq \infty$.
  Let $d \triangleq s - r + 1$, and let $V = (v_r, v_{r+1}, \ldots , v_s) \in
  \R^{p \times d}$ and $\hat{V} = (\hat{v}_r, \hat{v}_{r+1}, \ldots , \hat{v}_s)
  \in \R^{p \times d}$ have orthonormal columns satisfying $\Sigma v_j =
  \lambda_j v_j$ and $\hat{\Sigma}\hat{v}_j = \hat{\lambda}_j\hat{v}_j$ for $j =
  r, r + 1, \ldots , s$. Then there exists an orthogonal matrix $G \in
  \Orthogonal(d)$ such that
  \begin{equation}
    \| \hat{V}G - V \|_F \leq \frac{2^{3/2}\min(d^{1/2}\|\hat{\Sigma} - \Sigma\|_{op}, \|\hat{\Sigma} - \Sigma \|_F)}{\min(\lambda_{r} - \lambda_{r-1}, \lambda_{s+1} - \lambda_{s})}.
  \end{equation}
\end{thm}
With this result in hand, we are ready to prove Lemma
\ref{lem:alignment-eigvecs-trot}.
\begin{proof}[Proof of Lemma \ref{lem:alignment-eigvecs-trot}]
  The data matrices $\nQ$ and $\tQ$ are symmetric $dn \times dn$ matrices with
  eigenvalues $\lambda_1 \leq \ldots \leq \lambda_{dn}$ and $\noisy{\lambda}_1
  \leq \ldots \leq \noisy{\lambda}_{dn}$, respectively. From Theorem
  \ref{thm:spectral-relaxation-minimizer} we have that the $d$ \emph{normalized}
  eigenvectors corresponding to $\lambda_1, \ldots , \lambda_d$ of $\tQ$ and
  $\noisy{\lambda}_1, \ldots , \noisy{\lambda}_{d}$ are exactly
  $n^{-1/2}\trot\transpose$ and $n^{-1/2}
  {\eigvecs}\transpose$, respectively. Then, letting $r = 1$ and $s = d$ and
  applying Theorem \ref{thm:useful-dk-bound}, there exists an orthogonal matrix
  $G \in \Orthogonal(d)$ such that:
  \begin{equation}
    \frac{1}{\sqrt{n}} \| {\eigvecs} \transpose G - \trot \transpose \|_F \leq \frac{2\sqrt{2d}\|\nQ - \tQ\|_2}{\lambda_{d+1}(\tQ) - \lambda_d(\tQ)}.
  \end{equation}
  Multiplying both sides of this expression by $\sqrt{n}$, we have:
  \begin{equation}
    \| {\eigvecs} \transpose G - \trot \transpose \|_F \leq \frac{2\sqrt{2dn}\|\nQ - \tQ\|_2}{\lambda_{d+1}(\tQ) - \lambda_d(\tQ)}.
  \end{equation}
  Now, by definition $\dQ = \nQ - \tQ$. If we assume $\Graph$ is
  connected,\footnote{ It is not particularly restrictive to assume that
    $\Graph$ is connected. In the case that $\Graph$ is not connected, the
    estimation problem splits over the connected components of $\Graph$, and all
    of our results hold separately for each connected component.} from
  \cite[Lemma 8]{rosen2019se} we have that $\lambda_{d+1}(\tQ) > 0$. Since
  $\trot \in \ker(\tQ)$, we know that $\lambda_d(\tQ) = 0$ and the above expression
  simplifies to:
  \begin{equation}
    \| {\eigvecs} \transpose G - \trot \transpose \|_F \leq \frac{2\sqrt{2dn}\|\dQ\|_2}{\lambda_{d+1}(\tQ)}.
  \end{equation}
  Taking the transpose of the terms inside the norm gives the desired
  result.
\end{proof}

\section{Proof of the main results}\label{sec:proof-main-results}

In this appendix, we prove the main results, i.e. Theorem
\ref{thm:spectral-error}, Theorem \ref{thm:opt-error}, and Corollary
\ref{cor:rinit-vs-ropt}.

\subsection{An upper bound for the estimation error in Problem \ref{prob:spectral-relaxation}}\label{app:proof-spectral-error}

\begin{proof}[Proof of Theorem \ref{thm:spectral-error}]
  To simplify the subsequent derivation, we will assume without loss of generality that
  $\trot$ and $\eigvecs$ are the representatives of their orbits satisfying
  $\Oorbdist(\trot, \eigvecs) = \|\trot -
  \eigvecs\|_F$. Recall from the definition of $\Sorbdist(\trot, \Rinit)$ that:
  \begin{equation}
    \Sorbdist(\trot, \Rinit) = \min_{G \in \SO(d)} \| \trot - G \Rinit \|_F.
  \end{equation}
  Therefore, we have:
  \begin{equation}
    \begin{aligned}
    \Sorbdist(\trot, \Rinit)^2 &= \min_{G \in \SO(d)} \|\trot - G \Rinit \|_F^2 \\
                               &\leq \|\trot - \Rinit \|_F^2, \\
                               &= \sum_{i=1}^n \|\trot_i - \SOrounded{\eigvecs_i} \|_F^2, \label{eq:blockwise-init-error}
    \end{aligned}
  \end{equation}
  where in the last line we have used the fact that $\Rinit$ consists of the
  projections of individual $(d \times d)$ blocks of $\eigvecs$ onto $\SO(d)$.
  From Lemma \ref{lem:so-to-o-rounding}, we have that each of the $n$ summands
  above satisfies:
  \begin{align}
    \|\trot_i - \SOrounded{\eigvecs_i}\|_F^2 &\leq 4 \| \trot_i - \eigvecs_i \|_F^2. \label{eq:elementwise-bound}
  \end{align}
  This, in turn, gives a corresponding bound on the summation:
  \begin{equation}
    \begin{aligned}
    \sum_{i=1}^n \|\trot_i - \SOrounded{\eigvecs_i}\|_F^2 &\leq 4 \sum_{i=1}^n \| \trot_i - \eigvecs_i \|_F^2 \\
                                                          &= 4 \| \trot - \eigvecs \|_F^2.
    \end{aligned}
  \end{equation}
  Since, by hypothesis, $\eigvecs$ and $\trot$ are representatives of their orbits
  satisfying $\Oorbdist(\trot, \eigvecs) = \|\trot - \eigvecs\|_F$, we have:
  \begin{align}
    4 \| \trot - \eigvecs \|_F^2 &= 4 \Oorbdist\left(\trot, \eigvecs \right)^2.
  \end{align}
  Applying Lemma \ref{lem:alignment-eigvecs-trot}, we directly obtain:
  \begin{align}
    4 \Oorbdist\left(\trot, \eigvecs \right)^2 &\leq 4(2\sqrt{2dn})^2\frac{\|\dQ\|_2^2}{\lambda_{d+1}(\tQ)^2}.
  \end{align}
  In summary, we have:
  \begin{align}
    \Sorbdist(\trot, \Rinit)^2 &\leq 4(2\sqrt{2dn})^2\frac{\|\dQ\|_2^2}{\lambda_{d+1}(\tQ)^2}.
  \end{align}
  Taking the square root of both sides of the inequality in
  the last line gives:
  \begin{equation}
    \Sorbdist(\trot, \Rinit) \leq \frac{4\sqrt{2dn}\|\dQ\|_2}{\lambda_{d+1}(\tQ)},
  \end{equation}
  which concludes the proof.
\end{proof}

\subsection{An upper bound for the estimation error in Problem \ref{prob:rot-sync}}\label{app:proof-mle-bound}
  We begin following the arguments of \citet[Appendix D.4]{preskitt2018phase}.
  From the optimality of $\Ropt$ we have:
  \begin{equation}
    \begin{aligned}
    \tr(\nQ \trot\transpose \trot) &= \tr(\tQ \trot\transpose \trot) + \tr(\dQ \trot\transpose \trot) \\
                                   &\geq \tr(\tQ {\Ropt}\transpose \Ropt) + \tr(\dQ {\Ropt}\transpose \Ropt) = \tr(\nQ{\Ropt}\transpose\Ropt).
    \end{aligned}
  \end{equation}
  Since $\tr(\tQ \trot\transpose \trot) = 0$, we can rearrange the above
  expression to obtain:
  \begin{equation}
    \tr(\tQ {\Ropt}\transpose \Ropt) \leq \tr(\dQ \trot\transpose \trot) - \tr(\dQ {\Ropt}\transpose \Ropt).
  \end{equation}
  Using the fact that $\tr(\dQ \trot\transpose \trot) = \vect(\trot)\transpose (\dQ \otimes
  I_n) \vect(\trot)$ (and likewise for $\tr(\dQ {\Ropt}\transpose \Ropt)$), we
  have:
  \begin{equation}
    \begin{aligned}
    \tr(\tQ {\Ropt}\transpose \Ropt) &\leq \vect(\trot - \Ropt)\transpose (\dQ \otimes I_n) \vect(\trot + \Ropt) \\
                                     &\leq \|\vect(\trot - \Ropt)\|_2\|\dQ \otimes I_n\|_2 \|\vect(\trot + \Ropt)\|_2 \\
                                     &= \|\trot - \Ropt\|_F \|\dQ\|_2 \|\trot + \Ropt \|_F \\
                                     &\leq 2\sqrt{dn}\|\trot - \Ropt\|_F \|\dQ\|_2. \label{eq:ropt-tq-cost-bound}
    \end{aligned}
  \end{equation}
  In order to lower-bound the right-hand side of \eqref{eq:ropt-tq-cost-bound}
  in terms of the estimation error $\Sorbdist(\trot, \Ropt)$, we will make use
  of the following technical lemma of \citet{rosen2019se}:
  \begin{lem}[Lemma 11 of \citet{rosen2019se}]\label{lem:rosen-proj}
    Let $\rot \in \Orthogonal(d)^n \subset \R^{d \times dn}$ and furthermore let
    $M = \{WR \mid
    W \in \R^{d \times d}\} \subset \R^{d \times dn}$ be the subspace of
    matrices with rows contained in $\image(\rot\transpose)$. Then
    \begin{equation}
      \begin{aligned}
        &\proj_V: \R^{dn} \rightarrow \image(\rot\transpose) \\
        &\proj_V(x) = \frac{1}{n}\rot\transpose\rot x
      \end{aligned}
    \end{equation}
    is the orthogonal projection onto $\image(\rot\transpose)$ with respect to
    the $\ell_2$ inner product, and the map
    \begin{equation}
      \begin{aligned}
        &\proj_M: \R^{d \times dn} \rightarrow M \\
        &\proj_M(X) = \frac{1}{n}X\rot\transpose\rot
      \end{aligned}
    \end{equation}
    which applies $\proj_V$ to the rows of $X$ is the orthogonal projection onto
    $M$ with respect to the Frobenius inner product.
  \end{lem}
  Since $\ker(\tQ) = \image(\trot\transpose)$ and $\dim(\image(\trot\transpose))
  = d$, from Lemma \ref{lem:rosen-proj}, we have:
  \begin{equation}
    \tr(\tQ {\Ropt}\transpose \Ropt) \geq \lambda_{d+1}(\tQ)\| P \|_F^2, \label{eq:rosen-p-bound}
  \end{equation}
  where 
  \begin{equation}
    \begin{aligned}
      \Ropt &= K + P \\
      K &= \proj_M(\Ropt) = \frac{1}{n}\Ropt \trot\transpose \trot \\
      P &= \Ropt - \proj_M(\Ropt) = \Ropt - \frac{1}{n}\Ropt \trot\transpose \trot
    \end{aligned}\label{eq:orthogonal-projection}
  \end{equation}

  is an orthogonal decomposition of $\Ropt$ and the rows
  of $P$ are contained in the orthogonal complement of
  $\image(\trot\transpose)$

The following lemma provides a bound on $\Sorbdist(\trot, \Ropt)^2$ in terms of $\|P\|_F^2$.
\begin{lem}\label{lem:sorbdist-vs-p}
  Let $\Ropt$ and $\trot$ be representatives of their orbits such that
  $\Sorbdist(\trot, \Ropt) = \|\trot - \Ropt\|_F$, and $P = \Ropt -
  \proj_M(\Ropt)$ as defined in
  \eqref{eq:orthogonal-projection}. Then:
  \begin{equation}
    \frac{1}{4}\Sorbdist(\trot, \Ropt)^2 \leq \|P\|_F^2.
  \end{equation}
\end{lem}

\begin{proof}
  Let $X = \frac{1}{n}\trot {\Ropt}\transpose$, so that $K = X\transpose \trot$.
  Expanding the left hand side, we have:
  \begin{equation}
    \begin{aligned}
    \Sorbdist(\trot, \Ropt)^2 &=  \|\Ropt - \trot\|_F^2 \\
                              &\leq \|\Ropt - \SOrounded{X\transpose}\trot\|_F^2,
    \end{aligned}
  \end{equation}
  from the fact that the orbit distance is obtained as the minimum over $G \in
  \SO(d)$ of the quantity $\|\Ropt - G\trot\|_F$, and that by hypothesis this
  minimum is obtained as $\|\Ropt - \trot\|_F$. Breaking up the norm into its
  blockwise summands, and from the orthogonal invariance of the Frobenius norm,
  we can rearrange this expression as follows:
  \begin{equation}
    \begin{aligned}
    \|\Ropt - \SOrounded{X\transpose}\trot\|_F^2 &= \sum_{i = 1} \|\Ropt_i - \SOrounded{X\transpose}\trot_i \|_F^2 \\
                                                 &= \sum_{i = 1}^n \|\Ropt_i \trot_i\transpose - \SOrounded{X\transpose}\|_F^2.
    \end{aligned}
  \end{equation}
  From Lemma \ref{lem:so-to-o-rounding}, we know that each summand in the above
  expression satisfies
  \begin{equation}
    \| \Ropt_i \trot_i \transpose - \SOrounded{X\transpose} \|_F^2 \leq 4 \| \Ropt_i \trot_i \transpose - X\transpose \|_F^2.
  \end{equation}
  Since this bound is satisfied for each summand, the total summation satisfies
  \begin{equation}
    \begin{aligned}
      \sum_{i=1}^n \| \Ropt_i \trot_i \transpose - \SOrounded{X\transpose} \|_F^2 &\leq 4 \sum_{i=1}^n \| \Ropt_i \trot_i \transpose - X\transpose \|_F^2 \\
      &= 4 \sum_{i=1}^n \| \Ropt_i - X\transpose \trot_i \|_F^2 \\
    &= 4 \| \Ropt - X\transpose \trot \|_F^2.
    \end{aligned}
  \end{equation}
  Since $K = X\transpose \trot$, we have:
  \begin{equation}
    \begin{aligned}
    4\| \Ropt -X\transpose \trot \|_F^2 &= 4 \| \Ropt - K \|_F^2 \\
                                        &= 4 \| P \|_F^2,
    \end{aligned}
  \end{equation}
  which gives the desired bound.
\end{proof}

With this result, we are ready to prove Theorem \ref{thm:opt-error}.

\begin{proof}
  From \eqref{eq:rosen-p-bound} and \eqref{eq:ropt-tq-cost-bound}, we have:
  \begin{equation}
    \lambda_{d+1}(\tQ)\|P\|_F^2 \leq 2\sqrt{dn}\|\trot - \Ropt\|_F \|\dQ\|_2. \label{eq:cost-bound-p}
  \end{equation}
  Since, by hypothesis, $\Ropt$ and $\trot$ are the representatives of their
  orbits satisfying $\Sorbdist(\trot, \Ropt) = \|\trot - \Ropt\|_F$, from Lemma
  \ref{lem:sorbdist-vs-p} we have
  \begin{equation}
    \Sorbdist(\trot, \Ropt)^2 \leq 4 \|P\|_F^2. \label{eq:sorbdist-p-bound}
  \end{equation}
  Combining \eqref{eq:sorbdist-p-bound} with \eqref{eq:cost-bound-p}, we obtain:
  \begin{equation}
    \Sorbdist(\trot, \Ropt) \leq \frac{8\sqrt{dn}\|\dQ\|_2}{\lambda_{d+1}(\tQ)},
  \end{equation}
  which is what we intended to show.
\end{proof}

\subsection{An upper bound on $\Sorbdist(\Rinit, \Ropt)$}\label{app:spectral-mle-bound}

In this section, we give a proof of Corollary \ref{cor:rinit-vs-ropt}, bounding
the $\SO(d)^n$ orbit distance between the spectral initialization $\Rinit$ and
the maximum likelihood estimate $\Ropt$. First, we establish as the main
technical lemma a result that the orbit distances $\Sorbdist$ and $\Oorbdist$ on
$\SO(d)^n$ and $\Orthogonal(d)^n$ are \emph{pseudometrics}:

\begin{lem}[Orbit distances are pseudometrics]\label{lem:pseudometrics}
  The orbit distances $\Sorbdist$ and $\Oorbdist$ are \emph{pseudometrics} on
  $\SO(d)^n$ and $\Orthogonal(d)^n$, respectively. In particular, for all $X, Y,
  Z \in \SO(d)^n$, we have:
  \begin{enumerate}
    \item $\Sorbdist(X, X) = 0$ \label{item:identity}
    \item $\Sorbdist(X, Y) = \Sorbdist(Y, X)$ \label{item:symmetry}
    \item $\Sorbdist(X, Z) \leq \Sorbdist(X, Y) + \Sorbdist(Y, Z)$, \label{item:triangle}
  \end{enumerate}
  and likewise for $\Oorbdist$ on $\Orthogonal(d)^n$.
\end{lem}
\begin{proof}
  To simplify the subsequent derivation, we prove the result for the orbit
  distance $\Sorbdist$ on $\SO(d)^n$; the same argument applies \emph{mutatis
    mutandis} to $\Oorbdist$ on $\Orthogonal(d)^n$. A \emph{pseudometric} on
  $\SO(d)^n$ (resp. $\Orthogonal(d)^n$) is any nonnegative function $\SO(d)^n
  \times \SO(d)^n \rightarrow \R_{\geq 0}$ satisfying the properties
  \ref{item:identity}--\ref{item:triangle} \cite{kelley1955general}. To
  establish \ref{item:identity}, we have:
  \begin{equation}
    \Sorbdist(X, X) = \min_{G \in \SO(d)}\|X - GX\|_F = 0,
  \end{equation}
  since $\|A\|_F \geq 0$ for all $A$ and taking $G = I$
  realizes this minimum value.

  For \ref{item:symmetry}, we have:
  \begin{equation}
    \begin{aligned}
      \Sorbdist(X, Y) &= \min_{G \in \SO(d)} \| X - GY \|_F \\
      &= \min_{G \in \SO(d)} \|Y - G\transpose X\|_F
      &= \Sorbdist(Y, X),
    \end{aligned}
  \end{equation}
  where the second line follows from the orthogonal invariance of the Frobenius
  norm, and the last line follows from the fact that since $G\transpose = G^{-1}
  \in \SO(d)$, then $G\transpose$ ranges over all of $\SO(d)$ as $G$ does.

  Finally, to establish \ref{item:triangle}, we aim to prove that for any $X,
  Y, Z \in \SO(d)^n$:
  \begin{equation}
    \Sorbdist(X,Z) \leq \Sorbdist(X,Y) + \Sorbdist(Y,Z).
  \end{equation}
  Suppose the orbit distance $\Sorbdist(X,Y)$ is attained with minimizer
  $G^*_{XY} \in \SO(d)$ and likewise the distance $\Sorbdist(Y,Z)$ is attained with
  minimizer $G^*_{YZ} \in \SO(d)$. Define:
  \begin{equation}
    G' \triangleq G^*_{XY}G^*_{YZ}. \label{eq:gprime}
  \end{equation}
  Now, since $G'$ is itself the product of two elements of $\SO(d)$, we know $G'
  \in \SO(d)$, and therefore:
  \begin{equation}
    \Sorbdist(X,Z) = \min_{G \in \SO(d)} \|X - GZ\|_F \leq \|X - G'Z\|_F.
  \end{equation}
  Examining the right-hand side of this expression, we have:
  \begin{equation}
  \begin{aligned}
    \|X - G'Z\|_F &= \|X - G^*_{XY}Y + G^*_{XY}Y - G'Z\|_F \\
    &\leq \underbrace{\|X - G^*_{XY}Y\|_F}_{\Sorbdist(X,Y)} + \|G^*_{XY}Y - G'Z\|_F, \label{eq:triangle-expanded}
  \end{aligned}
  \end{equation} 
  where the last line follows from the triangle inequality for the Frobenius
  norm. Now, substitution of the definition \eqref{eq:gprime} into the second
  term of \eqref{eq:triangle-expanded} reveals:
  \begin{equation}
    \begin{aligned}
      \|G^*_{XY}Y - G'Z\|_F &= \|G^*_{XY}Y - G^*_{XY}G^*_{YZ}Z\|_F \\
      &= \|Y - G^*_{YZ}Z\|_F\\
      &= \Sorbdist(Y,Z),
    \end{aligned}
  \end{equation}
  where the second line follows from the orthogonal invariance of the Frobenius
  norm. Taken together, these results give:
  \begin{equation}
    \Sorbdist(X,Z) \leq \|X - G'Z\|_F \leq \Sorbdist(X,Y) + \Sorbdist(Y,Z),
  \end{equation}
  which is what we intended to show.
\end{proof}

Lemma \ref{lem:pseudometrics} suggests a straightforward proof of Corollary \ref{cor:rinit-vs-ropt}.
\begin{proof} 
  From the triangle inequality for $\Sorbdist$, we have:
  \begin{equation}
    \Sorbdist(\Rinit, \Ropt) \leq \Sorbdist(\trot, \Rinit) + \Sorbdist(\trot, \Ropt). \label{eq:rot-triangle-ineq}
  \end{equation}
  Substitution of \eqref{eq:spectral-error} and \eqref{eq:opt-error} into
  \eqref{eq:rot-triangle-ineq} gives the desired result.
\end{proof}

\section{Relationship to the method of \citet{moreira2021fast}}\label{app:moreira-comparison}

In their recent work, \citet{moreira2021fast} also propose an estimator for
pose-graph SLAM problems based on eigenvector computations. In this section,
we show that their approach is formally equivalent to the \emph{rotation-only}
variant of the spectral initialization we discuss in Section
\ref{sec:main-results} and therefore has estimation error satisfying the bound
\eqref{eq:ronly-bound}. \citet{moreira2021fast} specifically consider
\emph{unweighted} rotation measurements, which (from an estimation standpoint) is equivalent to considering the
generative model \eqref{eq:gen-model-ra} with \emph{identical} precisions (say $\kappa_{ij} = 1$) for all edges
$(i,j) \in \Edges$.

Their construction begins by considering the matrix $\moreiraR \in \R^{dn \times
  dn}$ with $d \times d$ block $i, j$ given by:
\begin{equation}
  \moreiraR_{ij} = \begin{cases}
    I_d & \text{if } i = j \\
    \nrot_{ij}, &  \edge \in \Edges \\
    0_{d\times d} &  \edge \notin \Edges.
  \end{cases}
  \label{eq:moreira-R}
  \end{equation}
  They observe that for all stationary points $\Rhat \in \SO(d)^n \subset \R^{d
    \times dn}$, there is a corresponding
  matrix $\Lambda \in \R^{dn \times dn}$ such that:
  \begin{equation}
    \underbrace{(\Lambda - \moreiraR)}_{\noisy{S}}\Rhat\transpose = 0,
  \end{equation}
  where $\Lambda$ has the symmetric $d \times d$ block diagonal structure:
  \begin{equation}
    \Lambda = \begin{bmatrix} \Lambda_1 & \cdots & 0 \\ \vdots & \ddots & \vdots \\ 0 & \cdots & \Lambda_n \end{bmatrix}.
  \end{equation}
  In the \emph{noiseless} case where $\moreiraR = \moreiratR$,\footnote{In
    keeping with the notation in the rest of this manuscript, we use the
    notation $\moreiratR$ to denote the measurement matrix \eqref{eq:moreira-R}
    constructed from the \emph{ground-truth} relative rotations $\trot_{ij}$.}
  the matrix $\true{S} = \Lambda - \moreiratR$ is given by \cite[Equation
  14]{moreira2021fast}:
  \begin{equation}\label{eq:moreiraS}
      \true{S} = (\moreiraL \otimes J_d) \circ \moreiratR,
  \end{equation}
  where $\moreiraL$ is the scalar (unweighted) rotational graph Laplacian with
  $i,j$ entry:
  \begin{equation}\label{eq:moreiraL}
    \moreiraL_{ij} = \begin{cases} \incEdges(i),& i = j, \\
      -1,& \edge \in \Edges, \\
    0,& \edge \notin \Edges,\end{cases}
  \end{equation}
  $J_d \in \R^{d \times d}$ is an all-ones matrix, and $\circ$ denotes the
  Hadamard product. Direct comparison of \eqref{eq:moreiraL} with
  \eqref{eq:laprotw} reveals that $\moreiraL$ is equivalent to $\LapRotW$ when
  $\kappa_{ij}=1$ for all $\edge \in \Edges$. Expanding \eqref{eq:moreiraS}, we
  have:
  \begin{equation}\label{eq:moreiraS-cases}
    \true{S}_{ij} = \begin{cases} \incEdges(i) I_d,& i = j, \\
      - \trot_{ij},& \edge \in \Edges,\\
    0_{d\times d},& \edge \notin \Edges \end{cases}.
  \end{equation}
  Comparing the definition of $\MeasRotConLap$ in \eqref{eq:lgrho} and
  $\true{S}$ in \eqref{eq:moreiraS-cases}, it is straightforward to verify that
  $\true{S} = \TrueRotConLap$ when $\kappa_{ij} = 1$. From the equivalence of
  $\true{S}$ and $\TrueRotConLap$, it follows that $\true{S} \succeq 0$ and
  $\trot\transpose \in \ker(\true{S})$, so the ground-truth rotations $\trot$
  can be recovered by computing the $d$ eigenvectors of $\true{S}$ corresponding
  to the smallest eigenvalues of $\true{S}$.\footnote{Recall from Section
    \ref{sec:spectral-init} that $\trot$ lie in $\ker(\TrueRotConLap)$ and from
    Section \ref{sec:problem-formulation} that $\TrueRotConLap \succeq 0$. The
    claim then follows from the equivalence of $\true{S}$ and $\TrueRotConLap$.}
  
  In the case of noisy measurements, \citet{moreira2021fast} propose to compute,
  as an approximation, the eigenvectors of $\noisy{S} = (\moreiraL \otimes J_3)
  \circ \moreiraR$, which has $d \times d$ blocks given by:
  \begin{equation}\label{eq:moreira-noisyS-cases}
    S_{ij} = \begin{cases} \incEdges(i) I_d,& i = j, \\
      - \nrot_{ij},& \edge \in \Edges,\\
      0_{d\times d},& \edge \notin \Edges. \end{cases}
    \end{equation}
    The justification given for this approximation is that, in the high
    signal-to-noise ratio regime, there ought to exist $R \in \SO(d)^n$ such
    that $\noisy{S}R \approx 0$. Once again, however, directly comparing
    definitions reveals that the quantity $(\moreiraL \otimes J_3) \circ
    \moreiraR$ is identical to $\MeasRotConLap$ with $\kappa_{ij} = 1$ (cf.\
    equations \eqref{eq:moreira-noisyS-cases} and \eqref{eq:lgrho}).
    Consequently, \citet{moreira2021fast}'s method is actually a
    \emph{particular instance} of the spectral estimator we propose in Section
    \ref{sec:spectral-init}, corresponding to the special case in which all
    rotational measurements have \emph{equal weights} and the translational
    measurements have been discarded (i.e.\ the rotation-only case discussed in
    Section \ref{sec:main-results}). Moreover, viewing this approach through the
    lens of the spectral relaxation in Problem \ref{prob:spectral-relaxation}
    provides formal justification for the method and allows us to derive the
    explicit performance guarantees given in this paper.



\bibliographystyle{plainnat}
\bibliography{references}

\end{document}